\documentclass[twocolumn]{article} 

\usepackage[american]{babel}

\usepackage{natbib} 
\usepackage{mathtools} 
\usepackage{booktabs} 
\usepackage{tikz} 
 \usetikzlibrary{arrows, shapes, positioning}
\RequirePackage{amsthm,amsmath,amsfonts,amssymb}
\usepackage{enumitem}
\usepackage[ruled,vlined]{algorithm2e}
\SetKwFunction{FIMB}{FindIMB}
\SetKwFunction{fcic}{FCItiers}
\SetKwFunction{OMB}{OMB}
\SetKwFunction{IMB}{IMB}
\SetKwFunction{FGESMB}{FGESMB}

\def\Equal{\texttt{=}}

\makeatletter
\newtheorem*{rep@theorem}{\rep@title}
\newcommand{\newreptheorem}[2]{%
\newenvironment{rep#1}[1]{%
 \def\rep@title{#2 \ref{##1}}%
 \begin{rep@theorem}}%
 {\end{rep@theorem}}}
\makeatother
\newtheorem{theorem}{Theorem}[section]
\newtheorem{lemma}[theorem]{Lemma}
\newreptheorem{theorem}{Theorem}

\counterwithin{lemma}{theorem}

\theoremstyle{remark}
\newtheorem{definition}[theorem]{Definition}
\newtheorem{assumption}[theorem]{Assumption}

\newcommand{\graph}[1] {\ensuremath{\mathcal {#1}}}
\newcommand{\vars}[1] {\ensuremath{\mathbf {#1}}}
\newcommand{\mb}{\ensuremath{\textnormal{MB}}}

\newcommand{\var}[1] {\ensuremath{#1}}

\newcommand\CI{{\,\perp\mkern-12mu\perp\,}}
\newcommand\nCI{{\,\not\mkern-1mu\perp\mkern-12mu\perp\,}}

\newcommand{\hyp}[2] {{\ensuremath{H_{\mathbf{#1}}^{#2}}}\;}
\newcommand{\mang} {{\ensuremath{\mathcal {G}_{\overline X}}}\;}
\newcommand{\imby} {{\ensuremath{\textnormal{MB}_{\overline X}(Y)}}\;}
\newcommand{\cmby} {{\ensuremath{\textnormal{\textbf{CMB}}_{\overline X}(Y)}}\;}
\newcommand{\mby} {{\ensuremath{\textnormal{MB}(Y)}}\;}
\newcommand\mydots{\ifmmode\ldots\else\makebox[1em][c]{.\hfil.\hfil.}\fi}

\title{Causal Markov Boundaries}

\author{
  Sofia Triantafillou
  \and
  Fattaneh Jabbari
  \and 
  Greg Cooper
  }
 \date{}

\begin{document}
\maketitle

\begin{abstract}
Feature selection is an important problem in machine learning, which aims to select variables that lead to an optimal predictive model. In this paper, we focus on feature selection for post-intervention outcome prediction from pre-intervention variables. We are motivated by healthcare settings, where the goal is often to select the treatment that will maximize a specific patient's outcome; however, we often do not have sufficient randomized control trial data  to identify well the conditional treatment effect. We show how we can use observational data to improve feature selection and effect estimation in two cases: (a) using observational data when we know the causal graph, and (b) when we do not know the causal graph but have observational and limited experimental data. Our paper extends the notion of Markov boundary to treatment-outcome pairs. We provide theoretical guarantees for the methods we introduce. In simulated data, we show that combining observational and experimental data improves feature selection and  effect estimation.
\end{abstract}

\section{Introduction}\label{sec:intro}
Feature selection is a fundamental problem in machine learning that aims to select the minimal set of features that lead to the optimal prediction of a target variable \var Y. For observational distributions, this set is the Markov boundary of \var Y, $\mb(Y)$. In causal graphical models, this set can be identified from the causal graph \graph G \citep{Pearl2000}. This set  exhausts the predictive information for the state of a variable \var Y, and can be used to obtain the best (and minimal) predictive model $P(Y|\mb(Y))$ for \var Y\cite{}. 

\begin{table*}[]
    \centering
    \begin{tabular}{|p{0.35\textwidth}|p{0.55\textwidth}|}
    \hline
        Observational Markov Boundary (OMB) of \var Y: \mb(Y) & The Markov boundary of \var Y. Leads to optimal prediction of \var Y from observational data. \\ \hline
    Interventional Markov Boundary (IMB) of \var Y relative to \var X: \imby & The Markov boundary of \var Y in the post-intervention distribution $P_X$. Leads to optimal prediction of $\var Y_x$ from experimental data.\\ \hline
    Causal Markov Boundaries  (CMB) of \var Y relative to \var X: \cmby& Sets of measured variables that satisfy  Definition \ref{def:cmb}. Possibly not unique, and possibly empty. If not empty, one of the CMBs leads to the optimal prediction of $Y_x$ from observational data.\\\hline
\end{tabular}
    \caption{Different Markov boundaries discussed in this paper.}
    \label{tab:mbs}
\end{table*}
In decision making, we are often interested in finding the optimal predictive model for the post-intervention distribution of an outcome \var Y after we intervene on a treatment \var X, when we only have observational data. Ideally, we would like to include in our model the  Markov boundary \vars Z of \var Y in the post-intervention causal graph that is parameterized with the post-intervention distribution. 
However, under causal insufficiency in which latent confounding may exist, the conditional post-interventional distribution $P(Y|do(X), \vars Z)$ may not be identifiable. For example, in Fig. \ref{fig:notID_CMBs}, $P(Y|do(X), A, B)$ is not identifiable from the observational distribution alone. In this case, we are interested in identifying the optimal set \vars Z for which the post-intervention distribution $P(Y|do(X), \vars Z)$ is identifiable from observational data, which we call the causal Markov boundary. 

Moreover, even when experimental data are available, they typically have much smaller sample sizes and are not powered to identify conditional distributions. In that case,  we would like to combine large observational data with limited experimental data to improve interventional feature selection and effect estimation.

Our methods are heavily motivated by embedded clinical trials \citep{angus2015fusing, angus2020remap}, which take place within usual clinical care. In these trials, patients who agree to participate are randomized to receive a treatment from among those considered effective for that patient. The electronic health records (EHRs) of the health system in which the trial is being conducted contains both experimental data from the trial, and observational data obtained outside (e.g., before/after) the trial, all measuring the same variables. Combining observational and experimental data has the potential to better predict the most effective treatments for individual patients, than either type of data alone.

Our contributions are the following:
\begin{itemize} [leftmargin=*]
\item We define the interventional Markov boundary $\imby$ and the  causal Markov boundaries \cmby for an outcome \var Y and a treatment \var X. These sets correspond to the minimal set of covariates \vars Z that are maximally informative for $\var Y|do(X)$, from experimental and observational data, respectively (Sec \ref{sec:mb}). Table \ref{tab:mbs} summarizes the types of Markov boundaries discussed in this paper.\\
\item  We present a Bayesian method that combines observational and experimental data to learn interventional Markov boundaries. The method provides estimates of the post-interventional distribution that are based on both observational and experimental data, when possible (Sec. \ref{sec:fimb}), in which case the IMB is  a CMB. In simulated data, we show that our method improves causal effect estimation  (Sec. \ref{sec:experiments}).
\end{itemize}

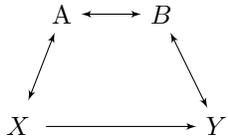
\begin{figure}[b]
\centering
\begin{tikzpicture}[>=latex']
\tikzstyle{every node}=[draw]
\node (X) [circle, draw =none] at (0,0){$X$};%
\node (phantom)[draw=none, above =1 of X]{};
\node (W1)[draw=none, right =0.2 of phantom]{A};

\node (W2)[draw=none, right =0.8 of W1]{$B$};
\node (Y)[draw=none, right =2 of X]{$Y$};

\path (X) edge[<->] (W1);
\path (W1) edge[ <->] (W2);
\path (W2) edge[<->] (Y);
\path (X) edge[ ->] (Y);
\end{tikzpicture}
\caption{\label{fig:notID_CMBs} An example SMCM. $\{X,A, B\}$ is the Markov boundary of $Y$ in \mang. $P(Y|do(X), A, B)$ is not identifiable from the observational distribution \var P, but $P(Y|do(X), A)$ and  $P(Y|do(X), B)$ are. $\{X, B\}$ is the causal Markov boundary for $Y_X$.}
\end{figure}

\section{Preliminaries}\label{sec:preliminary}
We use the framework of semi-Markovian causal models \citep[SMCMs, ][]{tian2003identification}, and assume the reader is familiar with related terminology. Variables are denoted in uppercase, their values in lowercase, and variable sets in bold. We use \graph G to denote a causal graph, and say \graph G induces a probability distribution \var P if \var P factorizes according to \graph G and the causal Markov condition.  

We use $Y|do(X)$ or $Y_X$ to denote a variable \var Y after the hard intervention on variable \var X. If we know the causal SMCM \graph G, a hard intervention of where a treatment $X$ is set to $x$  can be represented with the do-operator, $do(X\Equal x)$. We use $P_x$ or  to denote the interventional distribution over the same variables for $do(X\Equal x)$. In the corresponding graph, this is equivalent to removing all incoming edges into $\var X$, while keeping all other mechanisms intact.  We use  $\mang$ to denote the graph stemming from \graph G after removing edges into $X$. We use $\graph G_{\underline{X}}$ to denote the graph stemming from \graph G after removing edges out of \var X. We use the terms $Pa_\graph G(Z), Ch_\graph G(Z)$ to denote the set of parents and children of \var Z in \graph G, respectively. The set of variables that are connected with a variable \var Y through a bidirected path (i.e., a path that only has bidirected edges) is called the district of \var Y, and denoted $Dis_\graph G(Y)$.

\section{Markov Boundaries}\label{sec:mb}
A Markov blanket of a variable \var Y in a set of variables \vars V is a subset \vars Z of \vars V conditioned on which other variables are independent of \var Y: $\var Y \CI \vars V\setminus \vars Z|\vars Z$. The Markov boundary of \var Y is the Markov blanket that is also minimal (i.e., no subset of the Markov boundary is a Markov blanket) \citep{Pearl2000}. In distributions that satisfy the intersection property (including faithful distributions), the Markov boundary of a variable \var Y is unique \citep{Pearl1988}. To distinguish from other types of Markov boundaries defined in this work, we often use the terminology $\textbf{Observational Markov Boundary (OMB)}$ to denote the Markov boundary of a variable.

For a DAG \graph G, the OMB of a variable \var Y in any distribution faithful to \graph G is the set parents, children, and spouses of \var Y: $\mb(Y) = Pa_\graph G(Y)\cup Ch_\graph G(Y)\cup Pa_\graph G(Ch_\graph G(Y))$.
For SMCMs, it has been shown that the OMB of a variable \var Y is the set of parents, children, children’s parents (spouses) of \var Y, district of \var Y and districts of the children of \var Y, and the parents of each node of these districts \citep{richardson2003,Pellet}\footnote{\cite{Pellet} prove this for maximal ancestral graphs, but the proof can be readily adapted to SMCMs.}.

The OMB has been shown to be the minimal set
of variables with optimal predictive performance for a given  distribution and response variable, given some assumptions on the learner and the loss function \citep{tsamardinos2003towards}. In this work, we are interested in the model that gives  the optimal prediction of the post-intervention distribution, with the goal of designing optimal policies. For this reason, we are not interested in including post-intervention covariates in this model, because these variables are not known prior to treatment assignment, and thus, cannot affect the assignment. In the rest of this document, we make the following assumption:

\begin{assumption}\label{as:pretreatment}
Covariates \vars V are pre-treatment.
\end{assumption}

This simplifies the expressions for the  OMBs, because we no longer need to consider children of \var Y and their districts.
Knowing the OMB allows a more efficient representation of the conditional distribution of \var Y given \vars V, since the following equation holds:
\begin{equation}\label{eq:condprob}
    P(Y|\vars V) = P(Y|MB(Y)).
\end{equation}

\subsection{Interventional Markov Boundary}
Our goal is to identify the set of variables that lead to the optimal model for the post-intervention distribution of a target $Y$ relative to a specific treatment $X$.  We call this set the  \textbf{interventional Markov boundary (IMB)} of \var Y relative to \var X, and denote it \imby$\!$. Obviously, $\imby\subseteq\textnormal{MB(Y)}$. When we have data from the post-intervention distribution, we can apply statistical methods for OMB identification to obtain the IMB of \var Y relative to \var X. However, experimental data are often limited in sample sizes, while OMB identification methods may require large sample sizes.

If we know the causal graph \graph G, the  post-intervention distribution with respect to \var X is induced by the manipulated graph $\mang$. The IMB of \var Y is then the OMB of \var Y in $\mang$, and can be identified using the definition of the Markov Boundary above. However, the post-intervention distribution $P(Y|do(X), \imby\setminus X)$, may not be identifiable from the observational distribution. For example, in Fig. \ref{fig:notID_CMBs}, $\imby=\{X,A, B\}$, but $P(Y|do(X), A, B)$ is not identifiable from observational data. We then want to answer the following question: \emph{What is the best model for predicting $Y_X$ from the observational distribution, when the causal graph is known?}

\subsection{Causal Markov Boundaries}
To answer this question, we define the \textbf{causal Markov boundaries} of an outcome \var Y relative to a treatment \var X as follows:

\begin{definition}\label{def:cmb}
Let $\vars Z\subseteq (\vars V\cup X)$, and $\vars W=\vars Z\setminus X$. Then \vars Z is a causal Markov boundary (CMB) for \var Y relative to  \var X if it satisfies the following properties: 
\begin{enumerate}
    \item  $P(Y|do(X), \vars W)$ is identifiable from $P(X, Y, \vars V)$. 
    \item For every subset $\vars W'$ of $\vars V\setminus \vars W$ either $P(Y|do(X), \vars W, \vars W')$ $=P(Y|do(X),\vars W)$ or $P(Y|do(X), \vars W, \vars W')$ is not identifiable from $P(X, Y, \vars V)$.
    \item $\nexists \vars W' \subset \vars W$ s.t. $P(Y|do(X), \vars W') = P(Y|do(X), \vars W)$.
\end{enumerate}
\end{definition}
Condition (1) ensures that the post-intervention conditional probability of \var{Y_X} given a CMB is identifiable. Condition (2) states that the covariates that are not in that CMB are either redundant for the prediction of $Y_X$ given the CMB, or they make the post-intervention distribution non-identifiable. Condition (3) ensures that \vars Z is additionally maximally informative for $Y_X$ in the sense that you cannot remove any variable from \vars Z without losing some information for $Y_X$. This condition rules out sets like $\{X, A\}$ in Fig. \ref{fig:notID_CMBs}, where, while $P(Y|do(X), A)$ is identifiable from \var P, it is equal to $P(Y|do(X))$. Thus, conditioning on $\var A$ does not improve the prediction of $Y_X$ compared to its subset $\emptyset$. 

Notice that this definition does not capture the  spirit of Markov boundaries precisely: Markov boundaries make all remaining variables redundant for predicting $Y$; however, this does not necessarily hold with CMBs. For example, in Fig. \ref{fig:notID_CMBs},  $\{X, B\}$ is a CMB according to the definition above, but $\{A\}$ remains relevant for predicting $Y_X$; however, including it with $B$ in the CMB leads to non-identifiability.  

CMB is not necessarily unique; it is possible that multiple sets satisfy Definition \ref{def:cmb}. For example, assume the distribution \var P is induced by the SMCM shown in Fig. \ref{fig:CMBnds}. Both $\{X, B, C\}$ and $\{X, A, D\}$, satisfy Definition \ref{def:cmb}. The best predictive CMB for predicting $Y_X$ will depend on the parameters in \var P. We use the notation \cmby to denote the set of causal Markov boundaries of \var Y relative to \var X. Thus, we will generally need to find all CMBs and then determine which of them leads to the best prediction of $Y_X$.
Also, notice that the \cmby can be empty; thus, no subset of \vars V satisfies the Definition \ref{def:cmb}. This can happen for example if $X\rightarrow Y$ and $X\leftrightarrow Y$ in \graph G.

The CMB is useful in determining a minimal set of maximally predictive variables for which we can use observational data to predict post-interventional distributions.  In the next section we show that, for pre-treatment covariates, CMBs satisfy the back-door criterion and are subsets of the observational Markov boundary. These results enable more efficient algorithms for finding CMBs, limiting the types of estimators and the number of variable sets we need to consider.

\begin{figure}
\centering
\begin{tikzpicture}[>=latex']
\tikzstyle{every node}=[draw]
\node (X) [circle, draw =none] at (0,0){$X$};%
\node (A)[draw=none, above right =1 of X]{$A$};
\node (B)[draw=none, right =1 of A]{$B$};
\node (C)[draw=none, below right =1 of X]{$C$};
\node (D)[draw=none, right =1 of C]{$D$};
\node (Y)[draw=none, above right =1 of D]{$Y$};

\path (X) edge[<->] (A);
\path (A) edge[<->] (B);
\path (A) edge[->] (D);
\path (B) edge[ <->] (Y);
\path (X) edge[<->] (C);
\path (C) edge[<->] (D);
\path (D) edge[<->] (Y);
\path (C) edge[->] (B);
\path (X) edge[ ->] (Y);
\end{tikzpicture}
\caption{\label{fig:CMBnds} Causal Markov boundaries are not necessarily unique. Both $\{X, B, C\}$ and $\{X, A, D\}$ are causal Markov boundaries for \var Y relative to \var X.}
\end{figure}
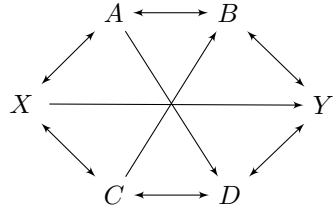
\subsubsection{Characterization}

Given a graph \graph G, \cite{shpitser2006a} provide a sound and complete algorithm (IDC) for estimating conditional post-intervention distributions from observational distributions induced by \graph G. The output of this algorithm  is an expression for $P(Y|do(X), \vars W)$ if the distribution is identifiable from distribution \var P and \graph G, or N/A otherwise. Thus, we can identify CMBs in a brute-force way by running IDC for every possible subset of \vars V, and then check for sets that satisfy the conditions in Def. \ref{def:cmb}. This process is  computationally expensive and would not be possible for graphs with more than a few variables. 

In this section, we provide theoretical results that lead to a much easier process when all candidate conditioning variables are pre-treatment (all proofs can be found in the supplementary). For pre-treatment covariates, one obvious family of sets for which the conditional post-intervention distributions are identifiable are sets  that m-separate $X$ and $Y$ in $\graph G_{\underline X}$. These sets satisfy Rule 2 of do-calculus \citep{Pearl2000}, so the conditional interventional distribution $P(Y|do(X), \vars W)$ is equal to the observational distribution $P(Y|X,\vars W)$. Sets of pre-treatment covariates that m-separate $X$ and $Y$ in $\graph G_{\underline X}$ are also known to satisfy the backdoor criterion \citep{van2014constructing} and the adjustment criterion \citep{shpitser2012validity}. However, these definitions are more general to include possible post-treatment covariates, and  are intended for estimating marginal post-intervention distributions (or average effects). For brevity, we will call sets that  m-separate $X$ and $Y$ in $\graph G_{\underline X}$ \textbf{backdoor sets}, since they block all back-door paths between \var X and \var Y. 

One question that arises is if there are sets that are not backdoor sets that may satisfy the conditions in Definition \ref{def:cmb}. In that case, identifiability could stem from some sequential application of do-calculus rules. As we show next, this is not possible for pre-treatment covariates. This ensures that we only need to check CMB Conditions (2) and (3) for sets for which $P(Y|do(X), \vars W) = P(Y|X,\vars W)$. This makes the identification of $P(Y|do(X), \vars W)$ more straightforward than having to compute more complex probability expressions.
 


 
\begin{theorem}\label{the:cmbbackdoor}
We assume that $P_x$ and $\graph G_{\overline X}$ are faithful to each other. If \vars Z is a CMB for \var Y relative to \var X, then $\vars Z\setminus X$ is a backdoor set for \var X relative to \var Y.
\end{theorem}

The second theoretical result is that any CMB of \var Y relative to \var X is a subset of the OMB of \var Y. While this sounds intuitive, it is not completely straightforward. It could be the case that conditioning on every subset of the OMB opens some m-connecting path between \var X and \var Y, that can only be blocked by a variable that is not a member of the Markov boundary. The following theorem proves that this is not possible, allowing for more efficient search algorithms:

\begin{theorem}\label{the:cmbsubmb}
We assume that $P_x$ and $\graph G_{\overline X}$ are faithful to each other. Every CMB \vars Z of an outcome variable  \var Y w.r.t a treatment variable \var X is a subset of the OMB $\mb(Y)$.
\end{theorem}
Based on Theorem \ref{the:cmbsubmb}, we only need to look for CMBs within subsets of the OMB of \var Y. 
So far, we have shown that both the IMB and any CMB are subsets of the OMB. We can also show that when the IMB is a CMB, then it also coincides with the OMB:

\begin{theorem}\label{the:imbcmbismb}
If \imby is a causal Markov boundary, then $\imby=\mb(Y)$.
\end{theorem}

\section{Combining observational and experimental data}\label{sec:fimb}
When the causal graph is known, we can obtain CMBs by looking for subsets of \mby that satisfy Def. \ref{def:cmb}. Unfortunately, in most real-world applications, the true graph is unknown, and selecting the causal/interventional Markov boundary is not possible from observational data alone. Experimental data may exist, but are typically much fewer than observational data, due to expense or  ethical concerns. This scenario is  common in embedded trials, where non-randomized patients are much more common than trial participants. 
In such cases, the experimental data may be under-powered to accurately estimate conditional effects. As a result, the conditional effects that can be derived from the experimental data have high variance and may not be reliable. In this case, combining all data (observational and experimental) in a Bayesian manner may help improve the prediction of $\var Y_x$.  

We assume that we have observational data $D_o$ and experimental data $D_e$ measuring treatment \var X, outcome \var Y, and pre-treatment covariates \vars V. We use $N_o, N_e$ to denote the number of samples in $D_o, D_e$, respectively.

We present a Bayesian method, called \FIMB, that uses both $D_e$ and $D_o$ to estimate the probability of a set being the $\imby\!$, and estimate $P(Y|do(X),\vars V) = P(Y|do(X),\imby\setminus X)$. The method is presented in Alg. \ref{algo:findIMB}. The method first estimates the OMB of \var Y in observational data \mby (Line 1), and then looks among subsets of \mby for sets that are IMBs (Line 2). It uses $D_e$ and $D_o$ to evaluate the probability that a set is an IMB (Line 3), and then returns a weighted average for $P(Y|do(X), \vars V)$ based on these probabilities (Line 5). 

The enabling idea of the method is that, when the IMB is a CMB, we can use both the $D_o$ and $D_e$ to estimate the conditional post-intervention distribution. Otherwise, we use only $D_e$ to derive the estimate. We use the following notation to express these hypotheses:
\begin{itemize}
    \item $\hyp{\vars Z}{c}$ is a binary variable denoting the hypothesis that \vars Z is the IMB $\imby$, and it is also  a CMB:  $\vars Z = \imby\wedge\vars Z\in \cmby$.
    \item $\hyp{\vars Z}{\overline{c}}$ is a binary variable denoting the hypothesis that \vars Z is the IMB $\imby$, but it is not a CMB: $\vars Z = \imby\wedge\vars Z\not\in \cmby$.
\end{itemize}
For a set $\vars Z^\star$, if either  $\hyp{\vars Z^\star}{c}$ or $\hyp{\vars Z^\star}{\overline c}$ is true, $\vars Z^\star$ is an IMB and therefore $P(Y|do(X), \vars V) = P(Y|do(X), \vars Z^\star\setminus X)$. Under $\hyp{\vars Z^\star}{c}$ though, $\vars Z^\star$ is also a CMB and therefore the pre- and post- intervention distributions are the same, i,e,
\begin{equation}\label{eq:ppequality}
    P(Y|do(X), \vars Z^\star\setminus X, \hyp{Z^\star}{c}) =  P(Y|X, \vars Z^\star\setminus X)
\end{equation}

In contrast, under \hyp{Z^\star}{\overline{c}}, Eq. \ref{eq:ppequality} does not hold: $\vars Z^\star$ is the IMB, but not a CMB. This means that $P(Y|do(X), \vars Z^\star\setminus X)$ is not identifiable from observational data, and we cannot use $D_o$ to estimate $P(Y|do(X), \vars Z^\star\setminus X)$\footnote{Notice however that $D_o$ may still place some constraints on $P(Y|do(X), \vars Z^\star\setminus X)$, like for example provide bounds.}. \emph{In summary, if either of $\hyp{Z^\star}{c} or \hyp{Z^\star}{\overline c}$ holds, $\vars Z^\star$ is the IMB and $P(Y|do(X), \vars V) = P(Y|do(X),  \vars Z^\star\setminus X)$. If $\hyp{Z^\star}{c}$ holds, we can use both $D_o$ and $D_e$ in our estimation of $P(Y|do(X),  \vars Z^\star\setminus X)$, while if $\hyp{Z^\star}{\overline c}$ holds we can only use $D_e$.} 

Based on this observation, we want to compute $P(\hyp{Z}{c}|D_e, D_o)$ and $P(\hyp{Z}{\overline c}|D_e, D_o)$ for possible IMBs \vars Z. These probabilities tell us both how likely it is that \vars Z is an IMB (their sum), and if we can include observational data in the estimation of $P(Y|\do(X), \vars V)$. Using  Bayes rule, we obtain: 
\begin{equation}\label{eq:score}
\begin{split}
& P(\hyp{Z^\prime}{c}|D_e, D_o) =\\ 
& \frac{P(D_e|D_o,\hyp{Z^\prime}{c} )P(D_o|\hyp{Z^\prime}{c})P(\hyp{\vars Z{^\prime}}{c})}{\displaystyle \sum_{\vars  Z}{\sum_{C=c, \overline c} P(D_e|D_o,\hyp{\vars Z}{C})P(D_o|\hyp{\vars Z}{C})P(\hyp{\vars Z}{C})}}.    
\end{split}
\end{equation}
We can similarly derive $P(\hyp{\vars Z}{\overline c}|D_e, D_o)$ by replacing each appearance of $c$  with $\overline c$ in the numerator. The denominator is the same for all sets. $P(\hyp{\vars Z{^\prime}}{c})$  and $P(\hyp{\vars Z{^\prime}}{\overline c})$ is our prior that $\hyp{\vars Z{^\prime}}{c}$ or $\hyp{\vars Z{^\prime}}{\overline c}$ holds. We set this to be uniform over both values of $C$ and all $\vars Z$.

As Eq. \ref{eq:score} shows, using Bayes rule, we can estimate the posterior probabilities for the set of hypotheses $\hyp{Z}{c}$ and $\hyp{Z}{\overline c}$ using marginal likelihoods of the experimental and observational data. In the next sections, we show how we can compute each term in Eq. \ref{eq:score}.

We present our results for multinomial distributions, but we believe the method can be readily extended to  any type of distribution with closed-form marginals. Due to space constraints, the closed-form solution for each equation appearing in remainder of the paper is presented in Supplementary Table \ref{tab:equations}. 

\textbf{Estimating $P(D_e|D_o,\hyp{Z^\prime}{c}),P(D_e|D_o,\hyp{Z^\prime}{\overline c}) $:}\\
Let $\vars W=\vars Z\setminus X$, and let $\theta_{Y_x|\mathbf W}$ be a set of parameters expressing the conditional probabilities for $P(Y|do(X), \vars W)$. Also, let $\theta_{Y|X, \mathbf W}$ denote the observational parameters for $P(Y|X, \vars W)$. By integrating over all $\theta_{Y_x|\mathbf W}$, we obtain 
\begin{equation}\label{eq:exp_score}
\begin{split}
    & P(D_e|D_o,\hyp{Z}{c}) =\\
    & \int_{Y_x|\mathbf W}P(D_e|\theta_{Y_x|\mathbf W})f(\theta_{Y_x|\mathbf W}|D_o, \hyp{Z}{c})d\theta_{Y_x|\mathbf W}
    \end{split}
\end{equation}
$f(\theta_{Y_x|\mathbf W}|D_o, \hyp{Z}{c})$ is the posterior for $\theta_{Y_x|\mathbf W}$ given the observational data, when \vars Z is the IMB and the CMB. In this case, $P(Y|do(X), \vars W) = P(Y|X, \vars W)$, and therefore $f(\theta_{Y_x|\mathbf W}|D_o, \hyp{Z}{c})=f(\theta_{Y|X, \mathbf W}|D_o)$. Eq. \ref{eq:exp_score} can then be rewritten in terms of the observational parameters as
\begin{equation}\label{eq:exp_score_obs}
\begin{split}
    &P(D_e|D_o,\hyp{Z}{c}) =\\
    &\int_{\theta_{Y|X, \mathbf W}}P(D_e|\theta_{Y|X,\mathbf W})f(\theta_{Y|X,\mathbf W}|D_o)d\theta_{Y|X,\mathbf W}
    \end{split}
\end{equation}
Eq. \ref{eq:exp_score_obs} is the marginal likelihood of $Y$ in experimental data, with parameter density $f(\theta_{Y|X,\mathbf W}|D_o)$ being equal to the parameter posterior given $D_o$. In other words, under $\hyp{Z}{c}$, the observational and experimental parameters coincide. Therefore, $D_o$ gives us a strong "prior" for $D_e$. Eq. \ref{eq:exp_score_obs} can be computed in closed-form for distributions with conjugate priors. 

Under \hyp{Z}{\overline c}, the equality of the observational and experimental parameters does not hold, and we cannot use the $\theta_{Y|X,\mathbf W}$ to inform $\theta_{Y_X|\mathbf W}$, at least not in a straightforward way. Instead, we impose that $f(\theta_{Y_X|\mathbf W}|D_o)= f(\theta_{Y_X|\mathbf W})$. Then  $P(D_e|D_o,\hyp{Z}{\overline c})$ corresponds to the marginal likelihood of \var Y in the experimental data, using a  prior that we model as being non-informative.

\textbf{Estimating $P(D_o|\hyp{Z}{c}),P(D_o|\hyp{Z}{\overline c}) $:}\\
These probabilities score how well the observational data fit with the hypotheses $\hyp{Z}{c}, \hyp{Z}{\overline c}$.
We can derive these terms on the basis of the OMB and its connection to the IMB and the CMBs. We first need to express the hypothesis \textbf{that a set \vars U is the OMB of \var Y}: Let  $\hyp{U}{o}$ denote this hypothesis; thus,  for any $\vars U\subseteq \vars V\cup X$,  $\hyp{\vars U}{o}$ is true iff  $\vars U$ is the OMB for \var Y. Then we can write 
\begin{equation}\label{eq:obs_score_c}
P(D_o|\hyp{Z}{C}) = \sum_{\vars U\subseteq \vars V\cup X} P(D_o|\hyp{U}{o})P(\hyp{U}{o}|\hyp{Z}{C}),
\end{equation}
for $C=c, \overline c$. Under $\hyp{Z}{c}\!$, Theorem \ref{the:imbcmbismb} implies that  $P(\hyp{U}{o}|\hyp{Z}{c})=1$  if $\vars U = \vars Z$, and zero otherwise. Under $\hyp{Z}{\overline c}$, the IMB is not  a CMB. Instead, the IMB has to be a subset of \vars U, therefore $P(\hyp{U}{o}|\hyp{Z}{C})=0$ for any $\vars Z\supset \vars U$.

$P(D_o|\hyp{U}{o})$ is the marginal likelihood of $Y$ in $D_o$, under the  hypothesis  that $\vars U$ is the data-generating OMB for \var Y in the observational data. We can obtain this likelihood using a Bayesian scoring algorithm like FGES, by scoring a DAG where \var Y is a child of variables  $\vars U$ (and no other edges are in the graph). We call this algorithm \FGESMB. For discrete variables, in the large sample limit this probability will be maximum only for the  true OMB:

\begin{theorem}\label{the:fgesmb}
Given dataset $D_o$ that contains samples from a strictly positive distribution $P$, which is a perfect map for a SMCM \graph G, the BD score~\citep{heckerman1995learning} will assign the highest score to the OMB of $Y$ in the large sample limit.
\end{theorem}

Eq. \ref{eq:obs_score_c} needs to be computed for all subsets of the covariate sets. In practice, however, for large $N_o$, these probabilities are dominated by the true OMB. Assuming our sample is large enough, we can use an algorithm with asymptotic guarantees for identifying the true OMB. In fact, once we commit to the an observational Markov boundary $\vars U^\star$,  Eq. \ref{eq:obs_score_c} leads to the following equations
\begin{equation}\label{eq:obs_score_c_fin}
\begin{split}
& P(D_o|\hyp{Z}{c}) \Equal P(D_o|\hyp{U^\star}{o})\textnormal{ for } \vars Z\Equal\vars U^\star\\
& P(D_o|\hyp{Z}{\overline c}) = P(D_o|\hyp{U^\star}{o})\textnormal{ for all } \vars Z\subseteq \vars U^\star\\
\end{split}
\end{equation}
For the remaining cases ($\hyp{Z}{c}$ and $\vars Z\Equal\vars U^\star$, or $\hyp{Z}{c}$ and $\vars Z\supset\vars U^\star$), the corresponding probabilities are zero.\\

\textbf{Bayesian estimation of $P(Y|do(X), \vars V, D_e, D_o)$:}\\
We now compute the $P(Y|do(X), \vars V)$ using Bayesian model averaging over the hypotheses $\hyp{Z}{C}$.  Let  $x,y,\vars V=\vars v$, denote given instances of $X,Y$ and $\vars V$, respectively.  When $\vars V=\vars v$, we use $\vars W=\vars w_v$ to denote the corresponding values of a set $\vars W\subset \vars V$. Recall that  under $\hyp{Z}{c}$ and under , $\hyp{Z}{\overline c}$,  $P(Y|do(X), \vars V) = P(Y|X, \vars W)$, where $\vars W=\vars Z\setminus X$.Then for a given instance of $\vars V=\vars v$, we have 
\begin{equation}\label{eq:prob_est}
\begin{split}
& P(y|do(x), \vars v, D_e, D_o) =\\ 
    & \sum_{\vars Z\subset \vars V}\sum_{C=c, \overline c}
    P(y|do(x), \vars w_v, D_e, D_o,\hyp{Z}{C})P(\hyp{Z}{C}|D_o, D_e)
\end{split}
\end{equation}
This equation computes the \emph{expectation} of the conditional probability parameter. The individual probabilities $P(Y|do(X), \vars W, D_e, D_o,\hyp{Z}{C})$ can be estimated as posterior expectations of $P(Y|do(X), \vars W)$ from the data. Specifically, under given \hyp{Z}{c}, $P(Y|do(X), \vars W) = P(Y|X, \vars W)$, and therefore we can use both $D_e$ and $D_o$ for the posterior expectation. In contrast, under \hyp{Z}{\overline c}, we only use $D_e$. Analytical equations for these probabilities for multinomial distributions can be found in Supplementary Table \ref{tab:equations}.
\begin{algorithm}[t]
\LinesNumbered
\SetKwInOut{Input}{input}
\SetKwInOut{Output}{output}
\SetKwFunction{FGESMB}{FGESMB}
\SetKwFunction{FIMB}{FindIMB}
\SetKwFunction{FCItiers}{FCIt-IMB}
\SetKwFunction{MarkovBoundary}{MarkovBoundary}
\Input{$D_o, D_e$, treatment \var X, outcome \var Y, pre-treatment covariates \vars V}
\Output{Post-intervention distribution $P(Y|do(X),\vars V)$}
$\mb(Y)\leftarrow\MarkovBoundary(Y, D_o)$\;
\ForEach{subset \vars Z of $\mb(Y)$ and $C=c, \overline c$}
{Compute $P(\hyp{\vars Z}{C}|D_e, D_o)$ using Eq. \ref{eq:score}\;
Compute $P(Y|do(X),\vars V, D_e, D_o,\hyp{\vars Z}{C})$ using  Eq. \ref{eq:prob_est}\;}
$P(Y|do(X),\vars V)\leftarrow \sum_\vars Z\sum_{C=c,\overline c}P(Y|do(X),\vars V, D_e, D_o, \hyp{\vars Z}{C})P(\hyp{\vars Z}{C}|D_e, D_o)$\;
\caption{FindIMB\label{algo:findIMB}}
\end{algorithm}

\begin{figure*}[t!]
\begin{tabular}{ccc}
\includegraphics[width =0.3\textwidth]{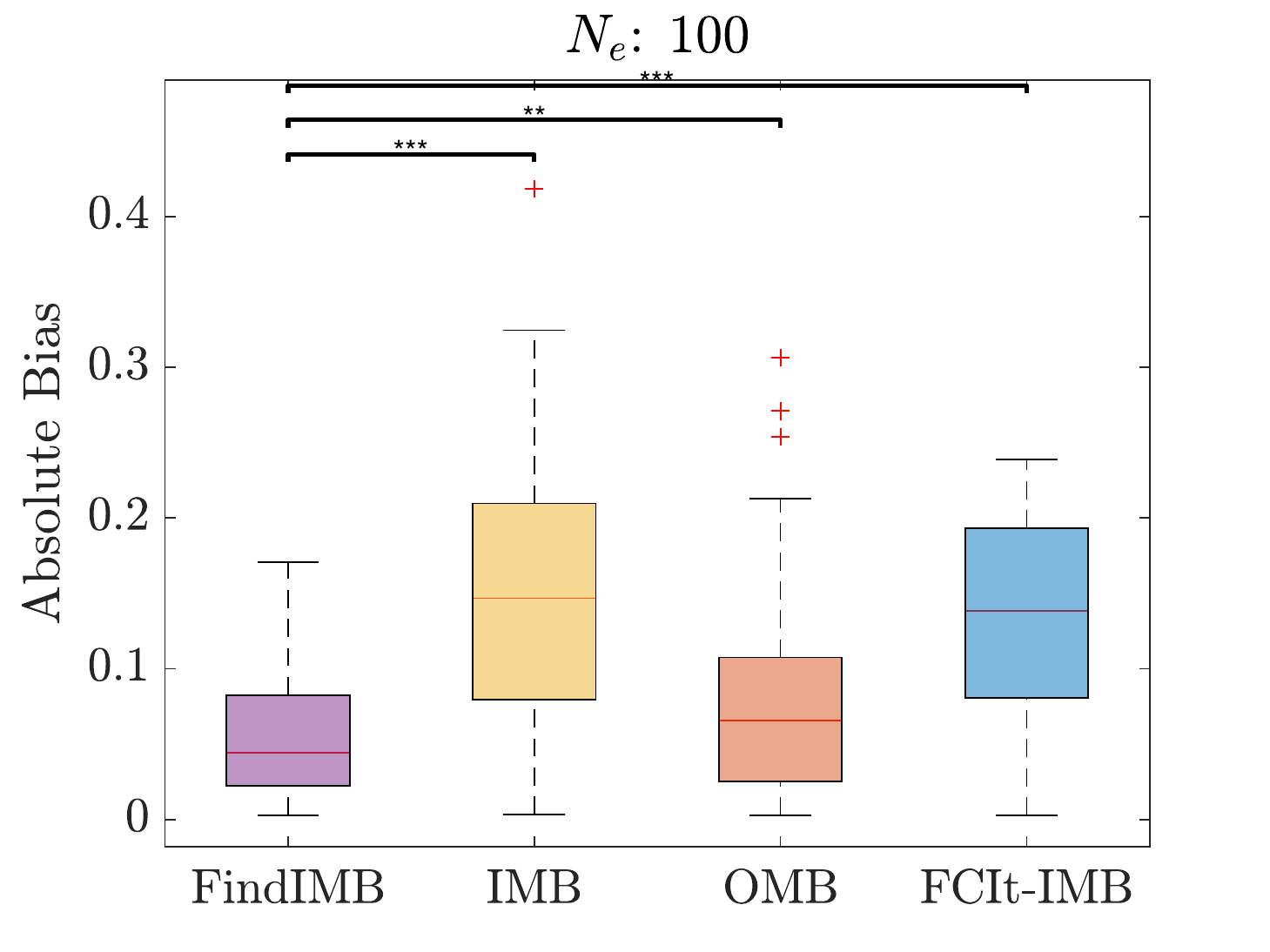}&
\includegraphics[width =0.3\textwidth]{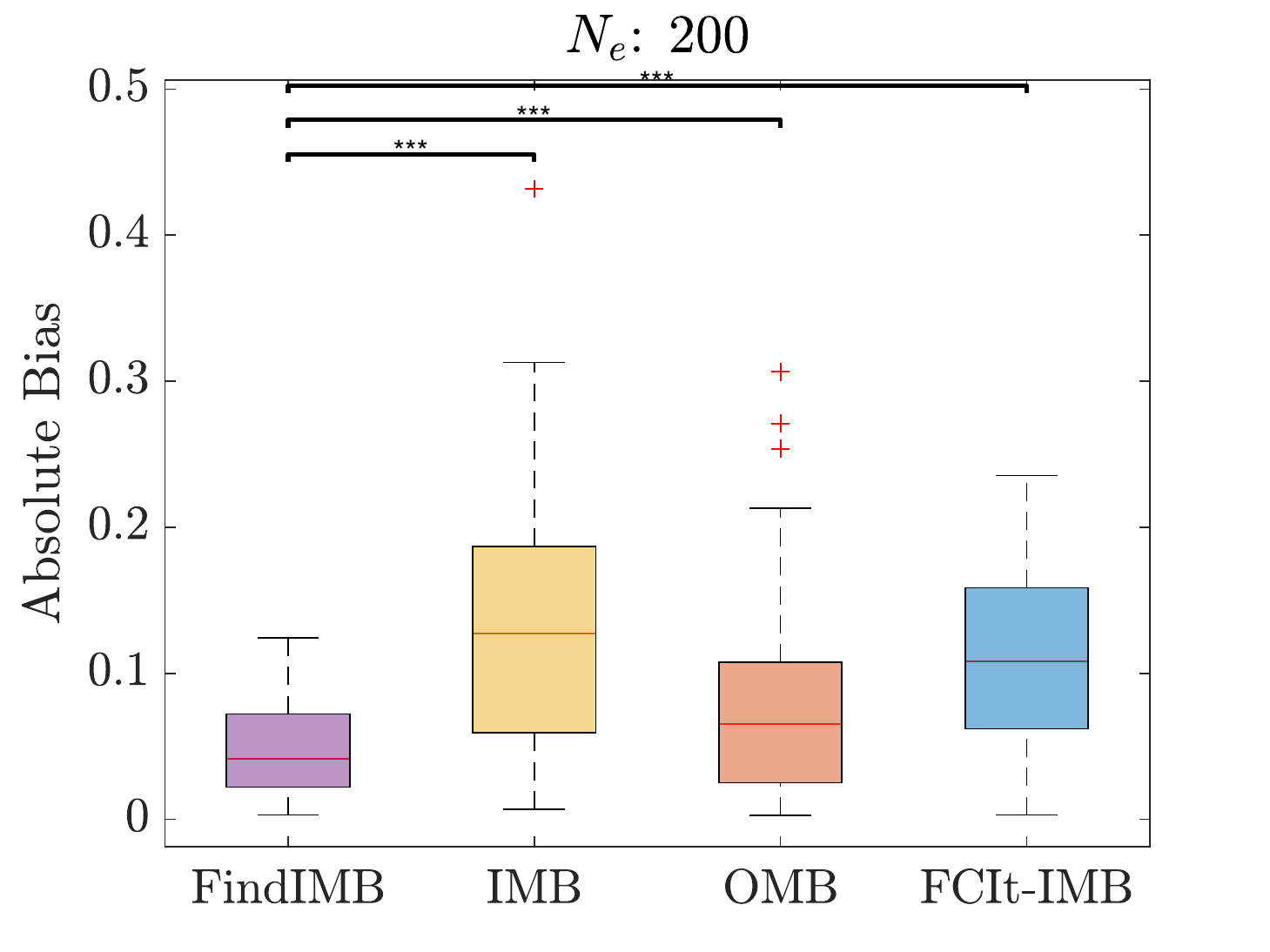}&
\includegraphics[width =0.3\textwidth]{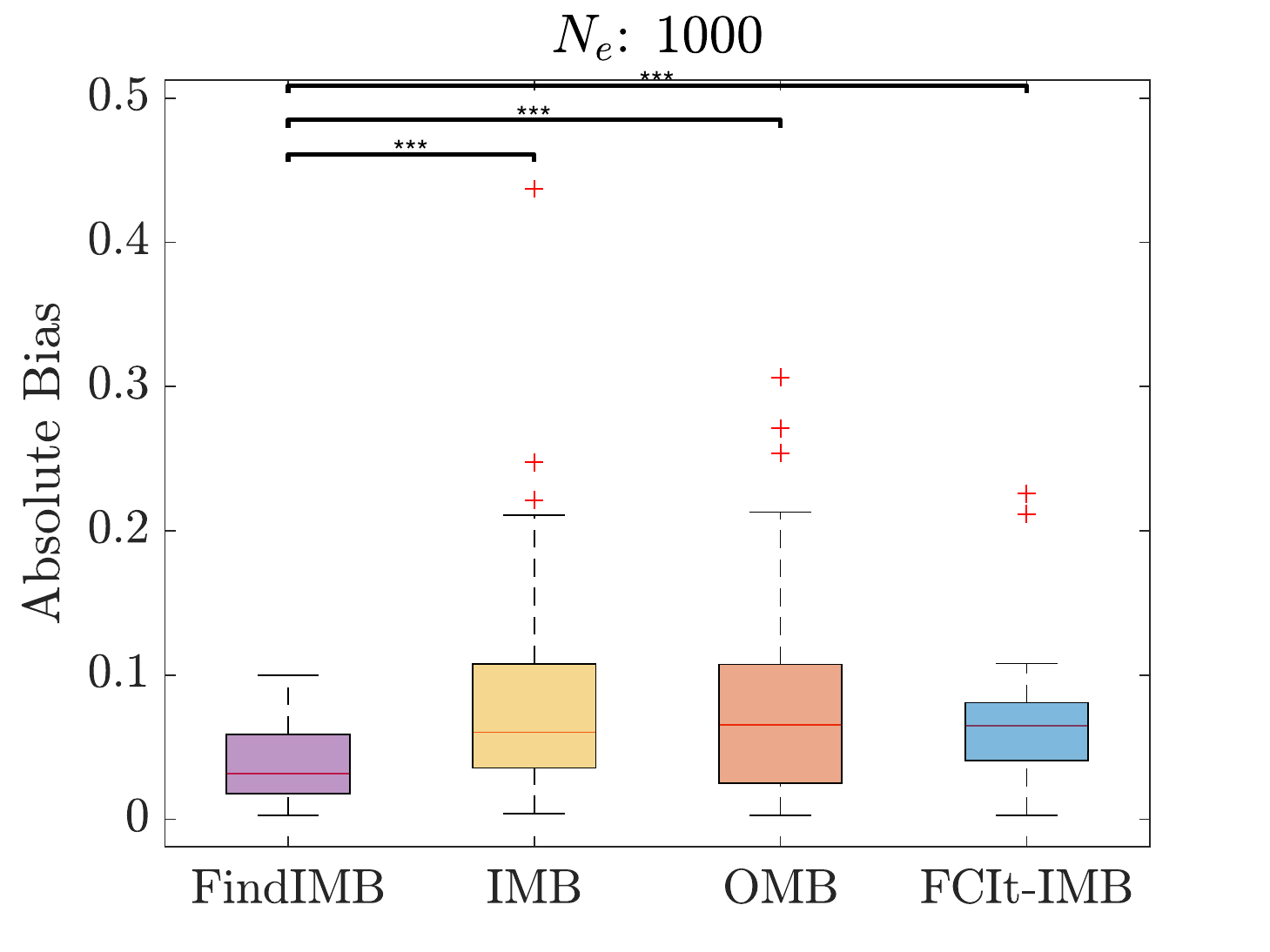}
\end{tabular}
\caption{\label{fig:results_1} Boxplots of absolute bias in the estimation of $P(Y|do(X),\vars V)$  using (a) \protect\FIMB (b) IMB (c) OMB and (d) \protect\FCItiers. Data were simulated from random DAGs with 10 observed and 5 latent variables. $D_o$ included 10,000 samples, and $D_e$ included 100 (left), 200 (middle) or 1000  (right) samples. \protect\FIMB improves the estimation of $P(Y|do(X), \vars V)$ particularly for smaller experimental sample sizes. Black asterisks denote statistical significance, assessed with the Wilcoxon signed-rank test. Three stars correspond to $p<0.001$.} 
\end{figure*}
\section{Related work}
We are not aware of other methods that try to identify causal and interventional Markov boundaries. Our work has connections and builds on work from many different areas.  Due to space constraints, we only focus on methods that do not require causal sufficiency.
\textbf{Markov boundaries:} Several algorithms learn OMBs from data under causal insufficiency \citep{yu2018, yu2020}. In addition, \FGESMB presented in Sec. \ref{sec:fimb} is also a sound and complete method for learning OMBs from data. These methods can be used to identify the IMBs from the experimental data, but they do not combine observational and experimental data to learn IMBs. 
\textbf{Identifiability:} \cite{shpitser2006a,shpitser2006identification} and \cite{tian2003identification} provide sound and complete identifiability results for post-intervention distributions from observational data when the causal graph is known. These methods can answer queries for a specific marginal or conditional probability of interest. \cite{hyttinen2015calculus} and \cite{jaber2019causal} provide similar identifiability results when the graph is unknown, using the Markov equivalence class of graphs that are consistent with the observational data.  \cite{hyttinen2015calculus} can provide identifiability results for graphs that are consistent with conditional independencies in both $D_e$ and $D_o$. However, the method is not proven to be complete for these settings. These methods are not directly comparable with  our method because they do not select features for optimal prediction. Moreover, they provide expressions for the post-intervention distributions that are based on observational data alone, not by combining $D_o$ and $D_e$  like \FIMB. \textbf{Combining observational and experimental data to learn causal graphs:} Several causal discovery methods combine observational and experimental data to learn causal structure \citep{triantafillou2015, hyttinen2014constraint, mooij2019, andrews2020}. These methods return a summarized version of all the causal graphs that are consistent with all the independence constraints in all the data sets, observational and experimental.  While these methods can be used to improve the estimation of IMBs, it is not clear that they can always provide a unique solution in this setting. Two additional drawbacks they have for the purpose of optimized target prediction are that (a) they rely on conditional independence tests that are unreliable when $N_e$ is low, and (b) they learn the entire graph and do not focus on finding the neighborhood of the target variable. This can result in unreliable orientations due to error propagation. The method that has the closest setting to ours is by \cite{andrews2020} (\cite{mooij2019} is also related, but more general, and the two are equivalent for our setting). The authors propose a method called FCItiers that can learn  a family of SMCMs from $D_e$ and $D_o$ when (a) the target of the intervention is known and (b) we specify "tiered knowledge" on the variables (e.g., we know which variables are pre-treatment).  The method is complete in these settings. In the experimental section, we develop a baseline comparison method based on FCItiers.
\textbf{Selecting optimal adjustment sets.} Some methods seek to select optimal adjustment sets for efficient average treatment effect estimation \citep{perkovic2017complete,rotnitzky2019efficient, rotnitzky2020}. While these methods have a  different purpose than ours, they have some connections with our work since, for pre-treatment variables, any CMB is also an adjustment set. We point out that while  optimal adjustment sets and CMBs may often coincide (for example in DAGs), they are not always the same (see example Fig \ref{sup:fig:adjustment} in the Supplementary). Moreover, these methods are not directly comparable to ours since they focus on identifying average treatment effects.
\textbf{Potential Outcomes Approaches:}  \cite{kallus2018removing} present a method for estimating conditional average treatment effects (CATEs) by combining $D_o$ and $D_e$. The method assumes a binary treatment and uses the experimental data to model the effect of possibly unmeasured confounders as a function of the measured covariates. The CATE is obtained from the $D_o$ by adding the modeled correction. The main assumption of the method is that the hidden confounding has an identifiable parametric structure. The method is implemented for continuous covariates and outcome and a linear correction function, obtained by solving a least squares optimization problem. It is not directly applicable to our settings of categorical covariates, and extending the optimization problem in these settings is not straightforward.
\textbf{Transportability:} Finally, our work has some connections with the field of  transportability  \cite{bareinboim2013general}, where knowledge of the causal graph is used to determine if the results of an experimental trial apply to a different population. However, the methods require knowing the causal graph, and focus on transferring estimators across distributions rather than combining data to improve estimators. 

\begin{figure*}[t!]
\begin{tabular}{ccc}
\includegraphics[width =0.3\textwidth]{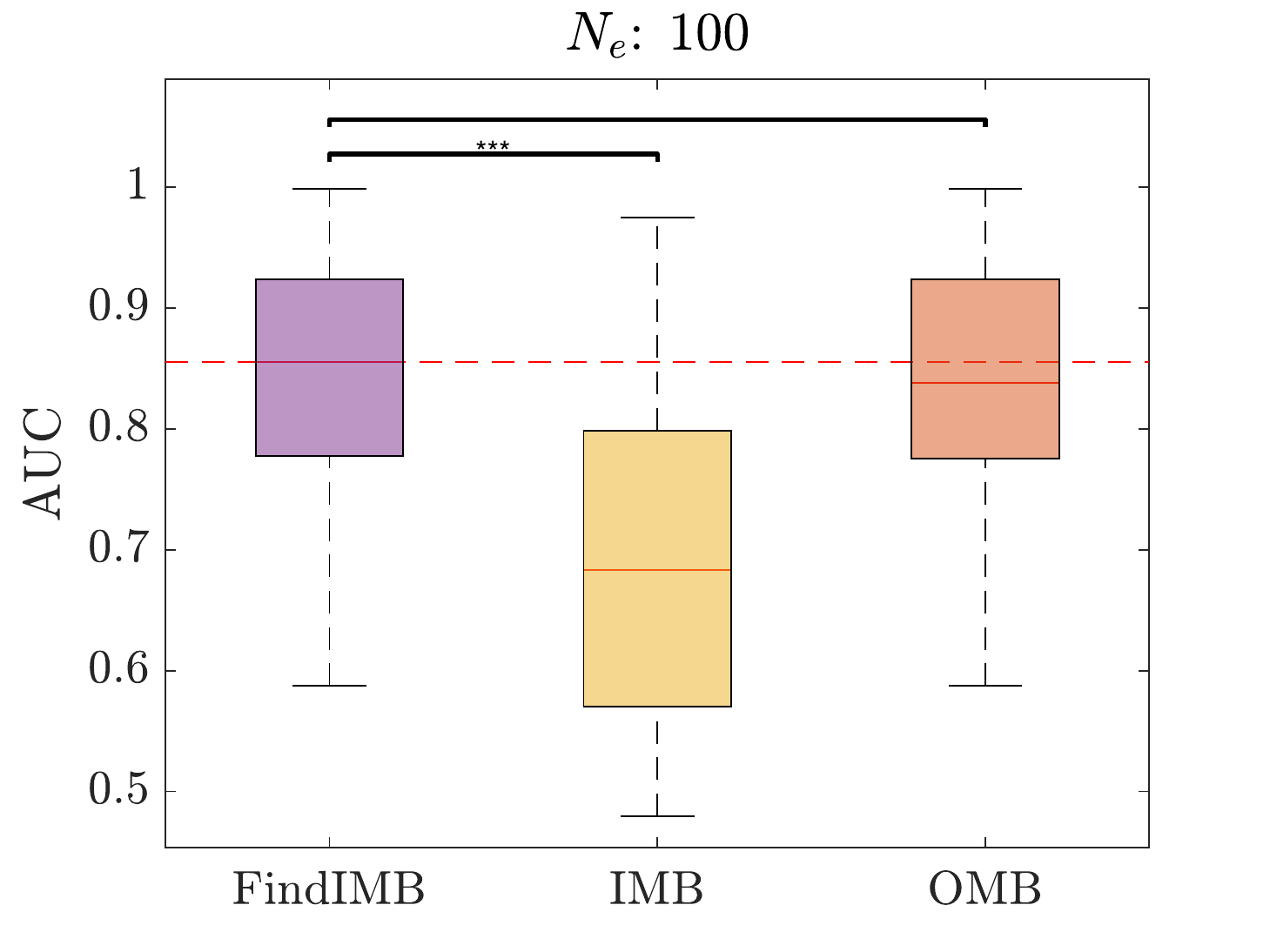}&
\includegraphics[width =0.3\textwidth]{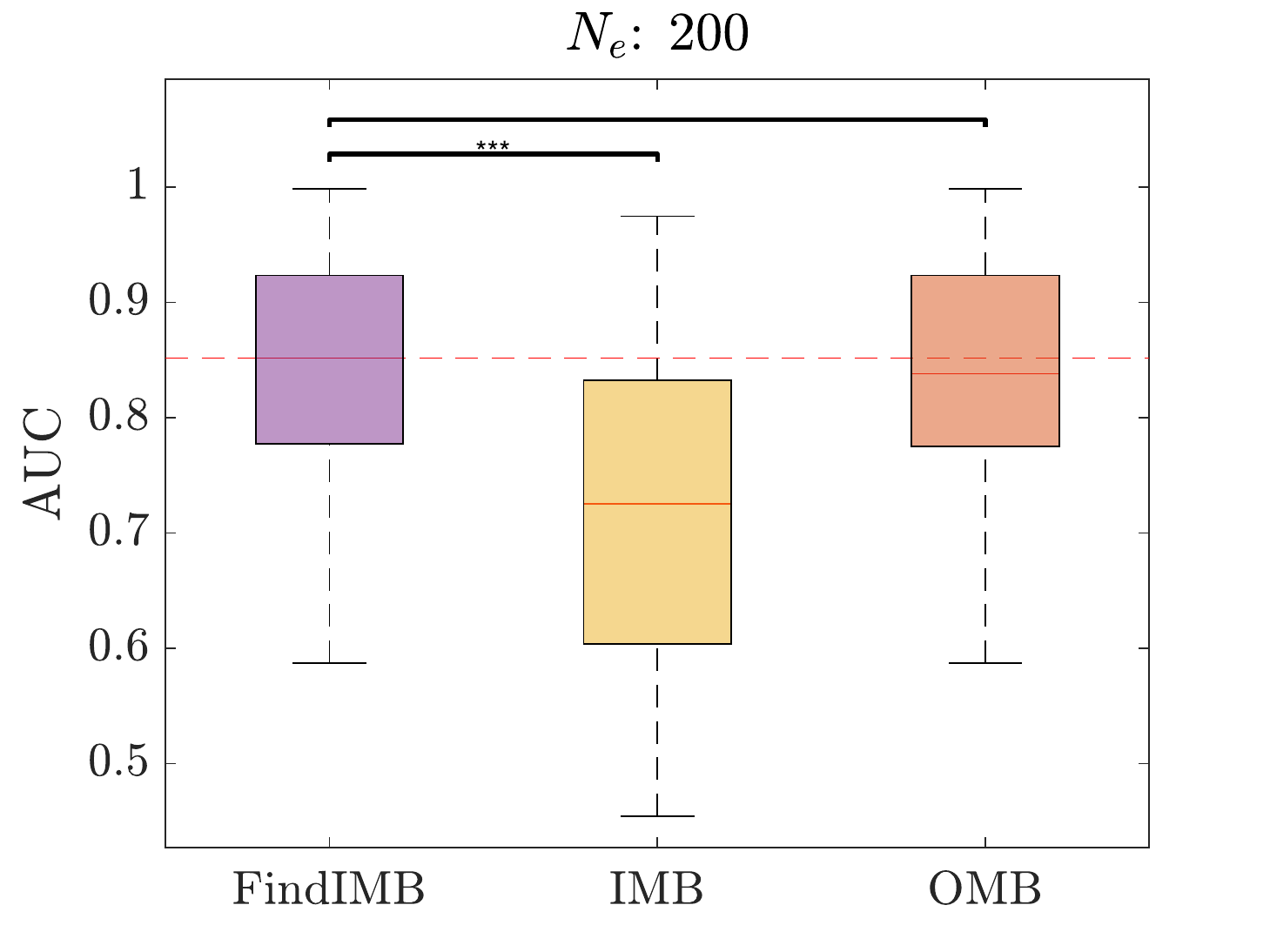}&
\includegraphics[width =0.3\textwidth]{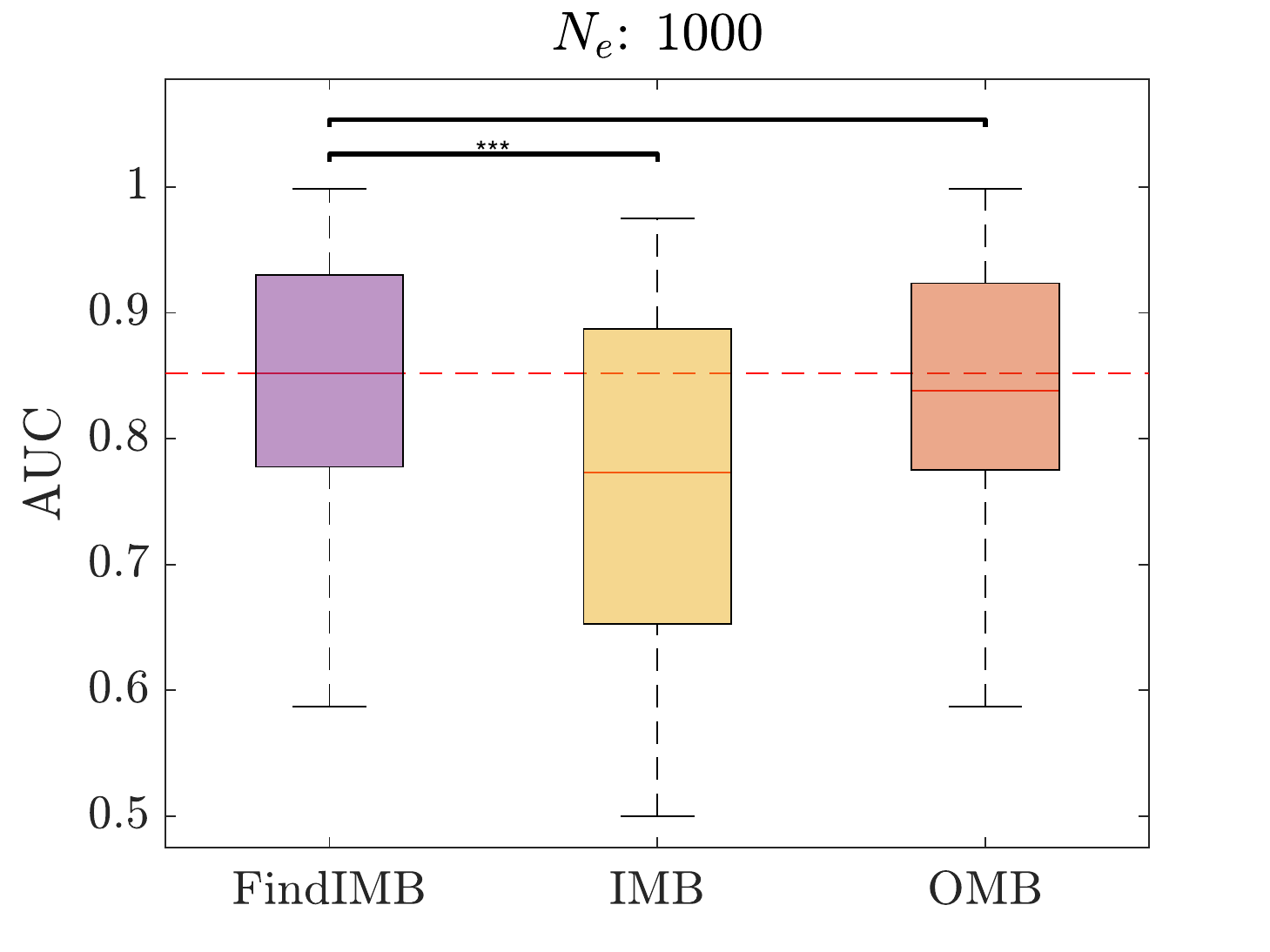}\\
\end{tabular}
\caption{ \label{fig:results_scaling}Boxplots of areas under the ROC curve for predicting $Y_x$ using (a) \protect\FIMB (b) IMB and (c) OMB. Data were simulated from random DAGs with 40 observed and 20 latent variables. $D_o$ included 10,000 samples, and $D_e$ included 100 (left), 200 (middle) or 1000  (right) samples. Black asterisks denote statistical significance, assessed with the Wilcoxon signed-rank test. Three stars correspond to $p<0.001$, no stars denote non-significance.} 
\end{figure*}

\begin{figure}[t!]
\centering
\includegraphics[width =0.8\columnwidth]{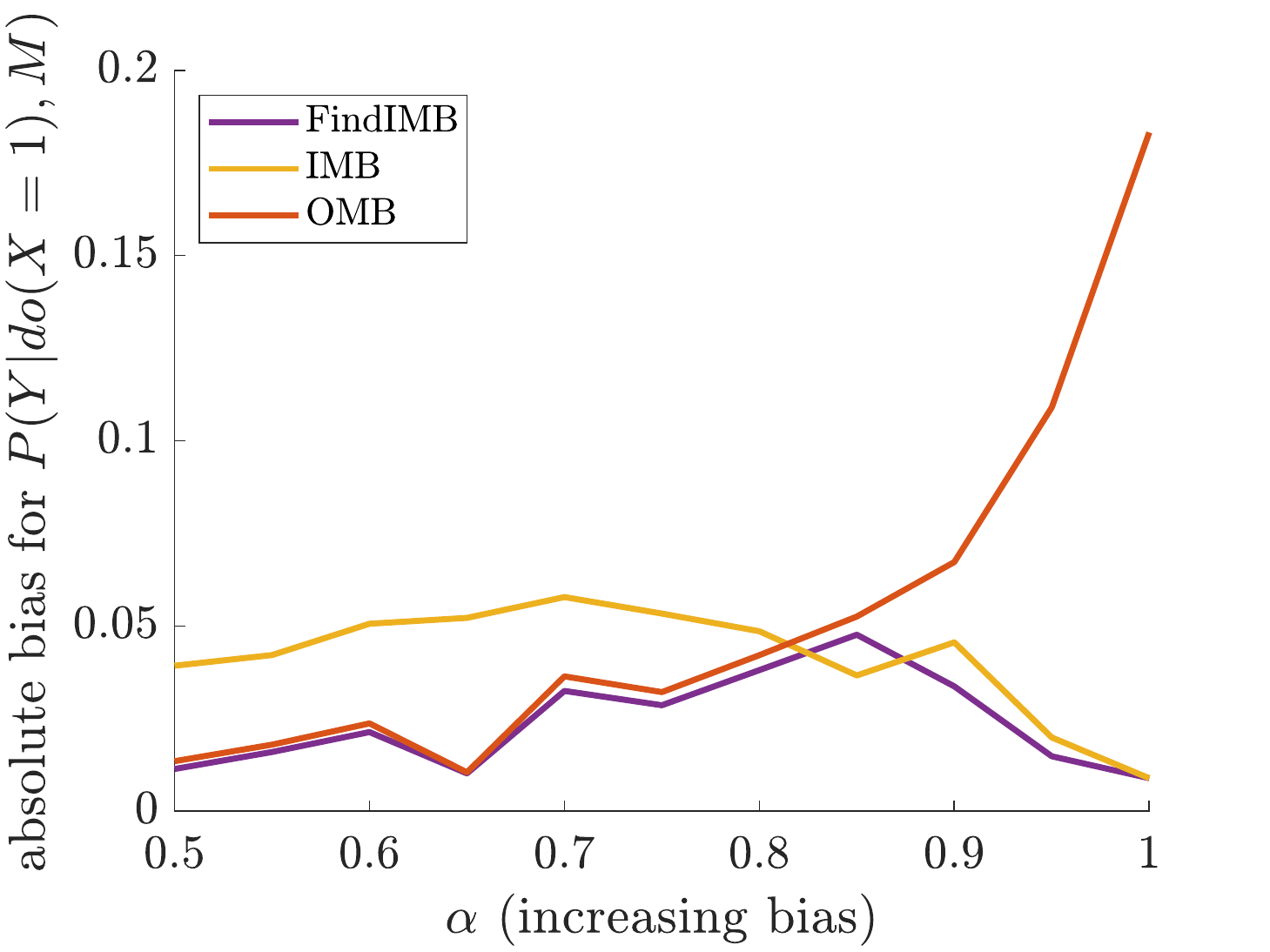}
\caption{\label{fig:results_bias_alpha} Performance of \protect\FIMB, IMB and OMB for estimating $P(Y|do(X\Equal 1), M)$ with increasing m-bias.} 
\end{figure}
\section{Experiments} \label{sec:experiments}
In this section, we show the performance of \FIMB  using simulated data.  We simulated random DAGs with a varying number of discrete variables, with mean in-degree 2. Each DAG includes a binary treatment \var X and outcome \var Y, where $X\rightarrow Y$. The remaining covariates \vars V are pre-treatment and are binary or ternary, and $1/3$ of the variables are set to be latent. The observational data $D_o$ consist of 10,000 simulated samples from the ground truth DAGs, and do not include values for the latent variables. 
\textbf{Comparison to other approaches.}  We compared \FIMB to the following approaches: 
(a) \textbf{IMB}: using only experimental data. We used $D_e$ to identify the \imby using  \FGESMB. After identifying  $\imby$, and we used the posterior expectation $P(Y|do(X), \imby\setminus X, D_e)$ as the estimator for $P(Y|do(X), \vars V)$. (b) \textbf{OMB}: using only observational data. We used $\FGESMB(D_o)$ to identify the OMB of $Y$, $MB(Y)$, and used  the posterior expectation $P(Y|do(X), MB(Y)\setminus X, D_o)$ estimated on $D_o$ as the estimator of $P(Y|do(X), \vars V)$. This estimator is unbiased when conditional ignorability holds for the OMB of $Y$. (c) \textbf{FCIt-IMB:} Using both observational and experimental data based on FCItiers: We use FCItiers using as input a data set $D$ constructed by concatenating $D_e$ and $D_o$, and adding a binary variable $I_e\rightarrow X$ that corresponds to the presence or absence of manipulation of $X$. So, $I_e = 1$ for samples in $D_e$ and $I_e=0$ otherwise.  \FCItiers outputs a PAG $\graph P$ representing all possible underlying SMCMs. Let $\graph P_{\overline X}$ denote the corresponding manipulated PAG. We then take the Markov boundary of $Y$ in $P_{\overline X}$ to be the $IMB_X(Y)$. 
After identifying  $IMB_X(Y)$, we test if it is a backdoor set in $\graph P_{\overline X}$. If so, we used both $D_e$ and $D_o$ pooled together to estimate $P(Y|do(X),IMB_X(Y)\setminus X)$. Otherwise, we only used $D_e$. 

First, we tested if our \FIMB method improves estimation of the probability $P(Y|do(X), \vars V)$. We simulated DAGs with 10 observed and 5 latent variables, and applied the methods described above. Each method outputs a set of variables $\vars Z$, that is used as an estimate for $P(Y|do(X), \vars Z, \vars V\setminus Z).$ Notice that even for 10 variables, the number of possible configurations of $\vars Z$ can be very large, and some of these configurations may be very rare. To avoid computing these parameters for all possible configurations, we tested the methods in a test dataset $D_e^{test}$, that includes  1000 treatment and 1000 control cases simulated from the manipulated ground truth graph. For each sample in $D_e^{test}$, we obtained an estimate $\hat P(Y|do(X), \vars V)$ with the four methods. The ground truth probability was estimated from the original manipulated Bayesian network with the junction tree algorithm. We then computed the average absolute bias $|\hat P(Y|do(X), \vars V)- P(Y|do(X), \vars V)|$ over all test samples. Fig. \ref{fig:results_1} shows that \FIMB  produced the most accurate probabilities, compared to using only observational or only experimental data. Moreover, \FIMB outperforms \FCItiers ($p<10^{-3}$ in all cases for a left-tailed t-test). One reason is that \FCItiers selects much larger IMBs than \FIMB, possibly due to error propagation that results in many bidirected edges. Thus, the resulting parameters are estimated based on much fewer samples. 

We also tested the scalability of \FIMB, using DAGs with 40 observed and 20 latent variables. For this experiment, we could not test against \FCItiers because the method results in very large IMBs. 
For the same reason, we could not estimate the true parameters $P(Y|do(X), \vars V)$ for computational reasons, since the ground truth IMBs can also be very large and include rare configurations. For this reason, instead of measuring the average absolute bias, we measured the performance of the methods to classify the test samples correctly. Fig. \ref{fig:results_scaling} shows the area under the ROC curve (AUC) of the models based \FIMB, IMB, and OMB. \FIMB performs on par or better than the two alternatives. Average running time for \FIMB (learning the model) was $1.71 \pm 2.46$ seconds per iteration.

One interesting finding is that in all experiments, using the observational data only, performs better than using experimental data and often is close to the performance when combining $D_e$ and $D_o$ using \FIMB. This happens because in random graphs, the effect of variables inducing bias is often negligible \citep{greenland2003quantifying}, and proxies of the unmeasured confounders are often included in the observed covariates. However, there are cases where the conditioning on an observed variable in observational data can produce heavily biased post-intervention probability estimates.  A very simple example is the graph in supplementary Fig. \ref{sup:fig:mbias}), which we call the ''m-bias" graph. To illustrate how m-bias can affect the prediction of $Y_x$ from observational data, we simulated data from the m-bias graph with binary variables. We set $Y, X $ and $M$ to be noisy-AND functions of their parents with a parameter $\alpha$. $\alpha$ has a monotonic relationship with the bias in estimating $P(Y|do(X=1), M)$ using observational data: Larger  $\alpha$ leads to a larger bias. We then varied alpha from $0.5$ to $1$, and we simulated $D_o$ and $D_e$ with 10,000 and 1000 samples, respectively. We used \FIMB, IMB and OMB to estimate $P(Y|do(X=1), M)$ in a test data set. Fig. \ref{fig:results_bias_alpha} shows the bias in the estimated parameter. We can see that while using $D_o$ to estimate $P(Y|do(X=1), M)$ leads to increasing bias, combining $D_e$ and $D_o$ can identify the situations where the parameter is not identifiable from observational data. We believe that noisy-AND types of distributions are not rare in biomedical data. 
\section{Discussion}
Our paper extends the concepts of Markov boundaries for predicting post-intervention distributions, and presents a method for learning such Markov boundaries from mixtures of observational and experimental data. The
method could be useful in settings like embedded trials, where we have abundant observational and limited experimental data. Future work includes extensions of the method to mixed data, overlapping covariate sets in observational and experimental data, and to non-singleton treatments and outcomes.
\bibliography{biblio.bib} 

\onecolumn

\begin{center}
\textbf{\Large  Supplementary Materials for: Causal Markov Boundaries}
\end{center}
\setcounter{equation}{0}
\setcounter{figure}{0}
\setcounter{table}{0}
\setcounter{page}{1}
\makeatletter
\renewcommand{\theequation}{S\arabic{equation}}
\renewcommand{\thefigure}{S\arabic{figure}}
\renewcommand{\thefigure}{S\arabic{table}}

\renewcommand{\bibnumfmt}[1]{[S#1]}
\renewcommand{\citenumfont}[1]{S#1}

\begin{table}[h!]
    \begin{tabular}{c|c} 
    Eq. number &Analytical Expression\\\hline
    Eq. \ref{eq:exp_score_obs}    &\(\displaystyle P(D_e|D_o,\hyp{Z}{c}) =\prod_{j=1}^q \frac{\Gamma(\alpha_j+N_j^o)}{\Gamma(\alpha_j+N_j^o+N_j^e)} \prod_{k=1}^r\frac{\Gamma(\alpha_{jk}+N_{jk}^o+N_{jk}^e)}{\Gamma(\alpha_{jk}+N_{jk}^o)}\)\\ \hline
       -  & \(\displaystyle P(D_e|D_o,\hyp{Z}{\overline c}) =\prod_{j=1}^q \frac{\Gamma(\alpha_j)}{\Gamma(\alpha_j+N_j^e)} \prod_{k=1}^r\frac{\Gamma(\alpha_{jk}+N_{jk}^e)}{\Gamma(\alpha_{jk})}\)\\ \hline
   Eq. \ref{eq:obs_score_c_fin} &\(\displaystyle P(D_o|\hyp{Z}{c}) = \prod_{j=1}^{\tilde q} \frac{\Gamma(\tilde \alpha_j)}{\Gamma(\tilde \alpha_j+\tilde N_j^o)} \prod_{k=1}^r\frac{\Gamma(\tilde \alpha_{jk}+\tilde N_{jk}^o)}{\Gamma (\tilde \alpha_{jk})}\)\\\hline
   Eq. \ref{eq:obs_score_c_fin}&\(\displaystyle P(D_o|\hyp{Z}{\overline c}) = \prod_{j=1}^{\tilde q} \frac{\Gamma(\tilde \alpha_j)}{\Gamma(\tilde \alpha_j+\tilde N_j^o)} \prod_{k=1}^r\frac{\Gamma(\tilde \alpha_{jk}+\tilde N_{jk}^o)}{\Gamma (\tilde \alpha_{jk})}  \)\\\hline
   Terms in Eq. \ref{eq:prob_est}&\(\displaystyle P(Y=k|x, Z=j, D_e, D_o, \hyp{Z}{c}) =\frac{N_{jk}^o+N_{jk}^e+\alpha_{jk}}{N_{j}^o+N_{j}^e+\alpha_{j}}  \)\\\hline
      Terms in Eq. \ref{eq:prob_est}&\(\displaystyle P(Y=k|x, Z=j, D_e, D_o, \hyp{Z}{\overline c}) =\frac{N_{jk}^e+\alpha_{jk}}{N_{j}^e+\alpha_{j}}  \)\\\hline
    \end{tabular}
    \caption{Closed-form solutions for Eq. \ref{eq:exp_score} -\ref{eq:prob_est}} in the main paper, for multinomial distributions with Dirichlet priors. Subscript $jk$ refers variable \var Y taking its $k$-th configuration, and variable set $\vars Z$ taking its $j$-th configuration. $\alpha_{jk}$ is the prior for the Dirichlet distribution. We set $\alpha_{jk}=1$ in all experiments. $N^o_{jk}, N^e_{jk}$ corresponds to counts in the data where $Y=k$ and  $\vars Z=j$ in $D_o$ and $D_e$, respectively. $N^o_{j}, N^e_{j}$ corresponds to counts in the data where $Z=j$. Tilde notation corresponds to the OMB $\vars U$.
    \label{tab:equations}
\end{table}

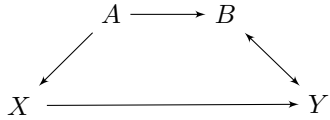
\begin{figure}[h!]
\centering
\begin{tikzpicture}[>=latex']
\tikzstyle{every node}=[draw]
\node (X) [circle, draw =none] at (0,0){$X$};%
\node (A)[draw=none, above right =1 of X]{$A$};
\node (B)[draw=none, right =1 of A]{$B$};
\node (Y)[draw=none, below right =1 of B]{$Y$};

\path (X) edge[->] (Y);
\path (A) edge[->] (X);
\path (A) edge[->] (B);
\path (B) edge[<->] (Y);
\end{tikzpicture}
\caption{\label{sup:fig:adjustment}An example where a CMB does not necessarily correspond to an optimal adjustment set \citep{henckel2019graphical}. $\cmby =\{\{X, A, B\}\}$, but the optimal adjustment set depends on the parameters.}
\end{figure}

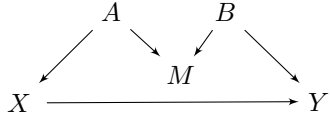
\begin{figure}[h!]
\centering
\begin{tikzpicture}[>=latex']
\tikzstyle{every node}=[draw]
\node (X) [circle, draw =none] at (0,0){$X$};%
\node (A)[draw=none, above right =1 of X]{$A$};
\node (B)[draw=none, right =1 of A]{$B$};
\node (M)[draw=none, below right =.5 of A]{$M$};
\node (Y)[draw=none, below right =1 of B]{$Y$};

\path (X) edge[->] (Y);
\path (A) edge[->] (X);
\path (A) edge[->] (M);
\path (B) edge[->] (Y);
\path (B) edge[->] (M);
\end{tikzpicture}
\caption{\label{sup:fig:mbias}The m-bias graph used to simulate data for Fig \ref{fig:results_bias_alpha}. 
\var A and \var B are unobserved. All variables were binary. Parameters were as follows: $P(A=1)=.8, P(B=1) = .8, P(M=1|A=1, B=1) =\alpha, P(X=1|A=1) = \alpha, P(Y=1|X=1, B=1)=\alpha$. All other parameters $P(Y=1|\dots),P(M=1|\dots), P(X=1|A=0)$ were set to zero.}
\end{figure}

\section{Proofs}
In this section, we provide a proof that every causal Markov boundary is backdoor set, which is defined below (Definition \ref{def:backdoorset}). We make the following assumptions throughout the entire document:
\begin{itemize}
    \item \var X causes \var Y
    \item all variables \vars V are pre-treatment.
\end{itemize}

\begin{definition}[Backdoor Set] \label{def:backdoorset}
\vars Z is a backdoor set for \var X, \var Y if and only if \vars Z m-separates \var X and \var Y in $\graph G_{\underline X}$.
\end{definition}
We use the following definitions from \citep{shpitser2006a}:

\begin{definition}[C-component]\label{def:C-component}
A C-component is as set of nodes \var S in \graph G where every two nodes are connected by a bidirected path.
\end{definition}

\begin{definition}[C-forest]\label{def:C-forest}
A graph \graph G where the set of all of its nodes is a C-component, and each node has at most one child is a C-forest. The set of nodes \vars R without children in the C-forest is called the root, and we say that \graph G is an \vars R-rooted C-forest. 
\end{definition}

C-forests are useful for defining hedges: 
\begin{definition}[hedge]
Let \vars X,\vars Y be sets of variables in \graph G. Let $F, F'$ be \vars R-rooted C-forests in \graph G such that $F'$  is a subgraph of \var F, \vars X only occurs in \var F, and $\vars R\in An(\vars Y)_{\graph G_{\overline X}}$. Then $F, F'$  form a hedge for $P(\vars Y|do(\vars X))$.
\end{definition}
The existence of a hedge for $P(\vars Y|do(\vars X))$ in \graph G is equivalent to the non-identifiability of $P(\vars Y|do(\vars X))$ (see Theorem 4 in \citep{shpitser2006a}).

\begin{lemma}\label{lem:ZnotBackdoor}
Let \vars Z be a set that is not a subset of any backdoor set (i.e., there exists no set $\vars Q\subseteq (\vars V\setminus \vars Z)$ such that $\vars Q\cup \vars Z$ m-separate $X$ and $Y$ in $\graph G_{\underline X}$). Then there exists in \graph G a bi-directed path from \var X to \var Y where every collider has a descendant in $\vars Z\cup Y$.  
\end{lemma}
\begin{proof}
The proof is a special case of Theorem 4.2 $(iv)\Rightarrow (ii)$ in \citep{richardson2002ancestral} with $\vars S \leftarrow \vars Z, \vars L\leftarrow \emptyset, \graph G \leftarrow\graph G_{\underline X}$. The proof is for ancestral graphs, but it is straightforward to show that it holds for SMCMs, given that every SMCM \graph G can be transformed to a maximal ancestral graph \graph M over the same nodes (by adding some edges) such that (a) \graph G and \graph M entail the exact same m-separations and m-connections and (b) the exact same ancestral relationships hold in both graphs. The theorem proves that if $\forall \vars Q\subseteq \vars (V\setminus \vars Z), \vars Z\cup \vars Q$ do not m-separate \var X and \var Y in $\graph G_{\underline X}$, then there exists a bidirected path between \var X and \var Y in  $\graph G_{\underline X}$ where every variable is an ancestor of some variables in $\vars Z\cup \{\var X, \var Y\}$, which means that there exists a path in \graph G a bi-directed path from \var X to \var Y where every collider has a descendant in $\vars Z\cup Y$ (since $\var X\rightarrow \var Y$ by assumption).
\end{proof}

\begin{lemma}\label{lemma:subsetofbd}
Let \vars Z be a set for which $P(Y|do(X), \vars Z)$ is identifiable from $P(Y|X, \vars Z)$, then \vars Z is a subset of a backdoor set.
\end{lemma}
\begin{proof}
First, notice that  $P(Y|do(X), \vars Z) =\frac{P(Y,\vars Z|do(X))}{P(Z|do(X))} = \frac{P(Y,\vars Z|do(X))}{P(Z)}$. Therefore $P(Y|do(X), \vars Z)$ is only identifiable if $P(Y, \vars Z|do(X))$ is identifiable. If \vars Z is not a subset of a backdoor set, then there exists a bidirected path where every variable has a descendant in $\vars Z\cup Y$ in \graph G by Lemma \ref{lem:ZnotBackdoor}. Let \graph F be the graph consisting of the bidirected path, and \graph F' be the same graph without \var X. Then \graph F, \graph F' are $\{Y,\vars Z\}$ rooted C-forests, and $\{Y,\vars Z\}\in An(\{Y,\vars Z\})$, so \graph F, \graph F' form a hedge for $\{Y, \vars Z\}$. Therefore, $P(Y,\vars Z|do(X))$ is not identifiable, and $P(Y|do(X), \vars Z)$ is not identifiable. 
\end{proof}
\begin{reptheorem}{the:cmbbackdoor}
We assume that $P_x$ and $\graph G_{\overline X}$ are faithful to each other. If \vars Z is a causal Markov boundary for \var Y relative to \var X, then $\vars W =\vars Z\setminus X$ is a backdoor set.
\end{reptheorem}
\begin{proof}
Assume \vars Z is a causal Markov boundary, but \vars W is  not a backdoor set. Since $P(Y|do(X), \vars W)$ is identifiable, by Lemma \ref{lemma:subsetofbd} \vars W is a subset of a backdoor set $\vars W\cup\vars Q$, where $\vars Q\subseteq (\vars V\setminus \vars W)$. Since by assumption \vars W is not a backdoor set, \vars Q is not the empty set (i.e., \vars W is a proper subset of a backdoor set). We will show that $P(Y|do(X), \vars W, \vars Q)\neq P(Y|do(X), \vars W)$. To show that, we only need to show that \vars Q is not independent from \vars W in \mang. Since \vars W is not a backdoor set, there exists a backdoor path from \var X to \var Y that is m$-$connecting given \vars W, but blocked given $\vars W\cup \vars Q$. Thus, some $Q\in \vars Q$ is a non-collider on that path, therefore \vars Q are not independent with \var Y given \vars W. Hence, $P(Y|do(X), \vars W, \vars Q)\neq P(Y|do(X), \vars W)$ and therefore \vars Z does not satisfy Condition (2), and \vars Z is not a causal Markov boundary (Contradiction).
\end{proof}

\begin{lemma}\label{lemma:mppath}
Let $\vars Z \subseteq \vars V$ be a backdoor set for $\var X, \var Y$, and let $\var Q\in (\vars Z\setminus\mb(Y))$ that  has an m-connecting path $Q\pi_{QY}Y$ with $\var Y$ given $\vars Z\setminus \var Q$. Then there exists a variable $W\in (\mb(Y)\setminus \vars Z)$ such that:  $\var W\cup \vars Z$ is a backdoor set and  $ W\nCI \var Y|\vars Z$ in  $\graph G_{\overline X}$.
\end{lemma}

\begin{proof}
Let \var Q be a variable as described above. Then there exists a variable $\var W\in\mb(Y)$ between \var Q and \var Y that is a non-collider on $\pi$, otherwise $\var Q\in Pa(Dis(Y))$, and therefore $\var Q\in \mb(Y)$. In addition, $\var W\not \in \vars Z$, otherwise $Q\pi_{QY}Y$ would be blocked given $\vars Z\setminus Q$.  We will now show, by contradiction, that adding \var W to the conditioning set \vars Z does not open any backdoor paths from \var X to \var Y; hence, $\vars Z \cup W$ is a backdoor set. 

Assume that conditioning on $\var W$ opens a path $\pi_{XY}$ between \var X and \var Y that is blocked given just $\vars Z$. Then $W$ must be a descendant of one or more colliders on that path. Let $C$ be the collider closest to $X$ on  $\pi_{XY}$ such that $C$ is blocked on $\pi_{XY}$ given \vars Z, but open given $\vars Z\cup W$. Then  $X \pi_{XC} C$ is open given \vars Z, and \var W is a descendant of \var C. Let $C\pi_{CW} W$ be the (possibly empty) directed path from $\var C$ to \var W, and let $W \pi_{WY} Y$ be the subpath of $\pi_{CY}$  from \var W to \var Y. Since $C$ is blocked on $\pi_{XY}$ given \vars Z,  no variable on $\pi_{CW}$ can be in \vars Z. But then   $X\pi_{XC} C \pi_{CW} W \pi_{WY} \ Y$ is an open path from \var X and \var Y given \vars Z in $\graph G_{\overline X}$.  Contradiction, since \vars Z is a backdoor set. Thus, \var W does not open any backdoor paths, and $\vars Z\cup W$ is also a backdoor set. 

Finally, $W$ is not independent of $Y$ given \vars Z in $\graph G_{\overline X}$, since $W \pi_{WY} Y$ is  open given \vars Z. 
\end{proof}

\begin{reptheorem}{the:cmbsubmb}
We assume that $P_x$ and $\graph G_{\overline X}$ are faithful to each other. Every causal Markov boundary \vars Z of an outcome variable  \var Y w.r.t a treatment variable \var X is a subset of the Markov boundary $\mb(Y)$.
\end{reptheorem}
\begin{proof}
We will show this by contradiction. Specifically, we will show that any set \vars Z that includes variables \vars Q  not in the Markov boundary of \var Y cannot satisfy one of the Conditions (2) or (3) of the causal Markov boundary.

Assume that \vars Z is a causal Markov boundary for \var Y with respect to  \var X.  and let $\vars W =\vars Z\setminus X$. Let $\vars Q = \vars W\setminus \mb(Y)$ be the  non-empty subset of \vars W that is not a part of the Markov boundary of \var Y. 

If there exists no $\var Q\in \vars Q$ that has an m-connecting path $Q\pi_{QY}Y$ to $\var Y$ given $\vars W\setminus \var Q$, then $\vars Q\CI Y|(\vars W\setminus \vars Q)$ in $\graph G_{\overline X}$. Conditioning on \var X cannot open any paths from \var X to \var Y;  therefore,
 $\vars Q\CI Y|X,(\vars W\setminus \vars Q)$ in $\graph G_{\overline X}$. Then by Rule 1 of the do-calculus \citep{Pearl2000}, $P(Y|do(X), \vars W) = P(Y|do(X), \vars W\setminus Q)$, and $\vars Z$ does not satisfy Condition (3) of the causal Markov boundary definition (Contradiction).
 
If there exists a $\var Q\in (\vars W\setminus\mb(Y))$ that has an m-connecting path $Q\pi_{QY}Y$ with $\var Y$ given $\vars Z\setminus \var Q$, then by Lemma \ref{lemma:mppath}, there exists a variable \var W in $\mb(Y)\setminus \vars Z$ such that $\vars Z\cup W$ is also a backdoor set, and $\var W\nCI Y|X,\vars Z$ in $\graph G_{\overline X}$. Then  $P(Y|do(X), \vars Z,W) \neq P(Y|do(X), \vars Z)$. Thus, \vars Z does not satisfy Condition (2) of the Causal Markov boundary definition (Contradiction).

Thus, \vars Z cannot include any variables that are not in the Markov boundary of \var Y.
\end{proof}

\begin{reptheorem}{the:imbcmbismb}
Let \graph G be a SMCM over \var X, \var Y,\vars V with \vars V occurring before \var X and \var Y. Let $\vars Z\subseteq \vars V\cup X$ be the IMB of \var Y relative to \var X. If \vars Z is a causal Markov boundary, then $\mb(Y)=\vars Z$.
\end{reptheorem}
\begin{proof}
$\imby\subseteq\mby$, so we need to show that $\mby\subseteq\imby$ when $\imby\in\cmby$. Assume that \vars Z is both the \imby and a causal Markov boundary, but there exists a variable \var Q in $\vars Z$ that is not in \mby. Then  \var Q  is reachable from \var Y through a bidirected path in \graph G but not in $\graph G_{\overline X}$. Since \graph G and $\graph G_{\overline X}$ only differ in edges that are into \var X, this path must be going through an edge that is incoming into  X. Thus, \graph G includes a bidirected path $Y\leftrightarrow \dots \leftrightarrow X$, and every variable on this path is in \imby =\vars Z. But then $\vars Z\setminus X$ cannot be a backdoor set, and by Theorem \ref{the:cmbbackdoor} \vars Z cannot be a causal Markov boundary. Contradiction. Thus, the Markov boundary of \var Y cannot include any more variables than $\vars Z$.
\end{proof}

\section{Convergence Proof for Observational Markov Boundary (OMB)}

\begin{definition}[Conditional Entropy] \label{def:condentropy}
Let $P$ be the full joint probability distribution over a set of variables \vars V, let $Y\in \vars V$ be a variable, and let $\vars Z \subseteq \vars V \setminus \{Y\}$ be a set of  variables. Then, the conditional entropy of \var Y given \vars Z is defined as follows \citep{cover1999elements}:
\begin{equation}\label{eq:entDisc}
        H(Y|\vars Z) = - \sum_{y}{\sum_{z}{P(y,z)\cdot \log P(y|z)}}
\end{equation}
where $y$ and $z$ denote the values of $Y$ and $\vars Z$, respectively. 
\end{definition}

\begin{lemma}\label{lem:ent}
Let $X,Y \in \vars V$ be two variables and $\vars Z  \subseteq \vars V \setminus \{X,Y\}$ be a set of  variables. Then, $H(Y|\vars Z) \geq H(Y|X, \vars Z)$, where the entropies are defined by Definition \ref{def:condentropy}, and the equality holds if and only if $Y \CI X | \vars Z$.
\end{lemma}
\begin{proof}
Applying the chain rule of entropy, the conditional mutual information can be computed as follows~\citep{cover1999elements}:
\begin{equation}
        I(X; Y|\vars Z) = H(Y|\vars Z) - H(Y|X,\vars Z)  \,.
\end{equation}
Given that the mutual information is nonnegative (i.e., $I(X; Y|\vars Z) \geq 0$) and $I(X; Y|\vars Z) = 0$ if and only if $Y \CI X |  \vars Z$ (see~\citep{cover1999elements}, page 29), it follows that:
\begin{equation}
\begin{split}
        H(Y|\vars Z) - H(Y|X,\vars Z) \geq 0\\
        H(Y|\vars Z)  \geq H(Y|X,\vars Z) 
         \,,
\end{split}
\end{equation}
where the equality holds if and only if $Y \CI X | \vars Z$. 
\end{proof}



For brevity, let $\vars V = \{\vars V \cup \var X \}$, where \var X is a treatment variable, and let \var Y be an outcome variable in the remainder of this section.

\begin{lemma}\label{lem:entMB}
All Markov blankets of $Y$ have the same entropy.
\end{lemma}
\begin{proof}
By definition, $\vars Z'$ is the Markov blanket of $Y$ if and only if $P(Y | \vars Z', \vars W) = P(Y | \vars Z')$ for any $\vars W \subseteq \vars V \setminus \vars Z'$, which indicates that $Y \CI \vars W | \vars Z'$. Also, according to Lemma \ref{lem:ent}, $H(Y|\vars Z') = H(Y|\vars Z', \vars W)$ for any $\vars W \subseteq \vars V \setminus \vars Z'$. Let $\vars Z$  also be a Markov blanket of $Y$. By multiple applications of Lemma \ref{lem:ent}, we obtain:
\begin{equation}
H(Y|\vars Z') = H(Y|\vars Z', \vars V \setminus \vars Z') = H(Y|\vars V)= H(Y|\vars Z, \vars V \setminus \vars Z')=H(Y|\vars Z)
\end{equation}
\end{proof}

\begin{lemma}\label{lem:entZneqZ'}
Let $\vars Z'$ be a Markov blanket of $Y$ and let $\vars Z$   be a set of variables that is not a Markov blanket of $Y$. Then, $H(Y | \vars Z') < H(Y | \vars Z)$, where the entropies are defined by Definition \ref{def:condentropy}.
\end{lemma}

\begin{proof}
Assume there is exists a set $\vars W \subseteq \vars V \setminus \vars Z$ such that $P(Y | \vars Z, \vars W) \neq P(Y | \vars Z)$. According to Lemma \ref{lem:ent} we have:
\begin{equation}
H(Y | \vars Z, W) < H(Y | \vars Z). 
\end{equation}
Also, given that $\vars V$ is a superset of $(\vars Z\cup \vars W)$, we have:
\begin{equation}
    H(Y | \vars V) \leq  H(Y | \vars Z, \vars W).
\end{equation}
Therefore,
\begin{equation}\label{eq:entVZ'}
    H(Y | \vars V) <  H(Y | \vars Z).
\end{equation}
Also, since $\vars Z'$ is a Markov blanket of $Y$, by Lemma \ref{lem:entMB} we have:
\begin{equation}\label{eq:entVZ}
    H(Y | \vars Z') =  H(Y | \vars V).
\end{equation}
Combining Equations (\ref{eq:entVZ'}) and (\ref{eq:entVZ}), we obtain:
\begin{equation}\label{eq:entzz'}
    H(Y | \vars Z') <  H(Y | \vars Z).
\end{equation}
\end{proof}

\begin{lemma}\label{lem:BDMB}
Given dataset $D_o$ that contains samples from a strictly positive distribution $P$, which is a perfect map for a SMCM \graph G, the BD score ~\citep{heckerman1995learning} for $\log P(D_o|\vars Z)$ is defined as follows in the large sample limit:
\begin{equation}\label{eq:BDMB}
\lim_{N \rightarrow \infty} \log P(D_o|\vars Z) = \lim_{N \rightarrow \infty} -N \cdot H(Y|\vars Z) - \frac{q \cdot (r-1)}{2} \log N + const. ,
\end{equation}
\end{lemma}
\begin{proof}
The BD score for $P(D_o|\vars Z)$ is calculated as follows~\citep{heckerman1995learning}:
\begin{equation}\label{eq:BD}
 P(D_o|\vars Z) =  \prod_{j=1}^{q}{\frac{\mathrm{\Gamma} (\alpha_{j})}{\mathrm{\Gamma} (\alpha_{j}+N_{j})} \cdot \displaystyle \prod_{k=1}^{r}\frac{\mathrm{\Gamma} (\alpha_{jk}+N_{jk})}{\mathrm{\Gamma} (\alpha_{jk})}} \,,
\end{equation}
where $q$ denotes instantiations of variables in $\vars Z$ and $r$ denotes values of variable $Y$. The term $N_{jk}$  is the number of cases in data in which variable $Y = k$ and its parent $\vars Z=j$; also, $N_{j}=\sum_{k=1}^{r}{N_{jk}}$.  The term $\alpha_{jk}$ is a finite positive real number that is called Dirichlet prior parameter and may be interpreted as representing ``pseudo-counts'', where $\alpha_{j} = \sum_{k=1}^{r}{\alpha_{jk}}$. BD can be re-written in $log$ form as follows:
\begin{equation}\label{eq:logIS}
 \log  P(D_o|\vars Z) =
 \sum_{j=1}^{q}{\left[{\log \Gamma (\alpha_{j})}-{\log\Gamma (\alpha_{j}+N_{j})} + \sum_{k=1}^{r}{\left[\log\Gamma (\alpha_{jk}+N_{jk})-\log\Gamma (\alpha_{jk})\right]}\right]}  .
\end{equation}
We can re-arrange the terms in Eq. (\ref{eq:logIS}) to gather the constant terms as follows:\vspace{-0.5mm}
\begin{equation}\label{eq:logIS2}
\begin{split}
    \log P(D_o|\vars Z) & = \sum_{j=1}^{q}{\left[ -{\log\Gamma (\alpha_{j}+N_{j})} + \sum_{k=1}^{r}{\log \mathrm{\Gamma} (\alpha_{jk} + N_{jk})}\right]} + \sum_{j=1}^{q}{\left[\log \mathrm{\Gamma} (\alpha_{j})  - \sum_{k=1}^{r}{\log \mathrm{\Gamma} (\alpha_{jk})}\right]} \\
    & = \sum_{j=1}^{q}{\left[ -{\log\Gamma (\alpha_{j}+N_{j})} + \sum_{k=1}^{r}{\log \mathrm{\Gamma} (\alpha_{jk} + N_{jk})}\right]} + const.
\end{split}
\raisetag{6ex}
\end{equation}

Using the Stirling’s approximation of $\lim_{n \rightarrow \infty} \log \mathrm{\Gamma}(n) = (n-\frac{1}{2}) \log (n) - n + const.$, we can re-write Eq. (\ref{eq:logIS2}) as follows:

\begin{equation}\label{eq:logIS_Strling}
\begin{split}
  & \lim_{N \rightarrow \infty} \log P(D_o|\vars Z) \\
  &= \lim_{N \rightarrow \infty} \sum_{j=1}^{q}\Biggl[ - (\alpha_{j}+N_{j} - \frac{1}{2}) \log (\alpha_{j}+N_{j}) + (\alpha_{j}+N_{j}) 
 + \sum_{k=1}^{r}{\left((\alpha_{jk} + N_{jk} - \frac{1}{2}) \log (\alpha_{jk} + N_{jk})- (\alpha_{jk} + N_{jk})\right)}  \Biggl] + const. \\
&= \lim_{N \rightarrow \infty}\sum_{j=1}^{q} \Biggl[ - \alpha_{j} \log (\alpha_{j} + N_{j}) - N_{j} \log (\alpha_{j} + N_{j})  + \frac{1}{2} \log (\alpha_{j} + N_{j}) + \alpha_{j} + N_{j}\\ & + \sum_{k=1}^{r}  {\left( \alpha_{jk} \log (\alpha_{jk} + N_{jk}) + N_{jk} \log (\alpha_{jk} + N_{jk}) - \frac{1}{2} \log (\alpha_{jk} + N_{jk}) - \alpha_{jk} - N_{jk}\right)}
 \Biggl] + const. \\
  &= \lim_{N \rightarrow \infty} \sum_{j=1}^{q} \Biggl[ - N_{j} \log (\alpha_{j} + N_{j})+ \sum_{k=1}^{r}  N_{jk} \log (\alpha_{jk} + N_{jk}) \Biggl] + \sum_{j=1}^{q} \Biggl[ -\alpha_{j} \log (\alpha_{j} + N_{j}) + \sum_{k=1}^{r} \alpha_{jk} \log (\alpha_{jk} + N_{jk}) \Biggl] \\
  &+ \frac{1}{2} \sum_{j=1}^{q} \Biggl[  \log (\alpha_{j} + N_{j})  - \sum_{k=1}^{r}  \log (\alpha_{jk} + N_{jk})
   + \alpha_{j}  + N_{j} - \sum_{k=1}^{r} \left(\alpha_{jk} + N_{jk}\right)\Biggl] + const. \\
  &= \lim_{N \rightarrow \infty} \sum_{j=1}^{q} \Biggl[ - N_{j} \log (\alpha_{j} + N_{j})+ \sum_{k=1}^{r}  N_{jk} \log (\alpha_{jk} + N_{jk}) \Biggl]
  + \sum_{j=1}^{q} \Biggl[ -\alpha_{j} \log (\alpha_{j} + N_{j}) + \sum_{k=1}^{r} \alpha_{jk} \log (\alpha_{jk} + N_{jk}) \Biggl] \\
  &+ \frac{1}{2} \sum_{j=1}^{q} \Biggl[  \log (\alpha_{j} + N_{j})  - \sum_{k=1}^{r}  \log (\alpha_{jk} + N_{jk})
 \Biggl] + const.
 \end{split}
\raisetag{6ex}
 \end{equation}
 In the last step of Eq. (\ref{eq:logIS_Strling}), we used the facts that 
$\sum_{k=1}^{r}{N_{jk}}=N_{j}$ and $\sum_{k=1}^{r}{\alpha_{jk}}=\alpha_{j}$, and we  applied these identities again to that equation to obtain the following: 
\begin{equation}\label{eq:logISStrling2}
\begin{split}
 & \lim_{N \rightarrow \infty} \log  P(D_o|\vars Z) =\\
 & \lim_{N \rightarrow \infty} \sum_{j=1}^{q} \sum_{k=1}^{r} \Biggl[ N_{jk} \log (\frac{\alpha_{jk} + N_{jk}}{\alpha_{j} + N_{j}}) + \alpha_{jk} \log (\frac{\alpha_{jk} + N_{jk}}{\alpha_{j} + N_{j}}) \Biggl] 
+ \frac{1}{2} \sum_{j=1}^{q}  \Biggl[ \log (\alpha_{j} + N_{j}) - \sum_{k=1}^{r}  \log (\alpha_{jk} + N_{jk})  \Biggl] + const. 
 \end{split}
 \end{equation}

\noindent Given that
 \begin{equation*}
     \lim_{N \rightarrow \infty} \frac{\alpha_{jk} + N_{jk}}{\alpha_{j}+N_{j}} = \frac{N_{jk}}{N_{j}}
 \end{equation*}
and 
 \begin{equation*}
     \lim_{N \rightarrow \infty} \sum_{j=1}^{q} \sum_{k=1}^{r} \alpha_{jk} \log (\frac{\alpha_{jk}+N_{jk}}{\alpha_{j}+N_{j}}) = const. ,
 \end{equation*}
in the limit, Eq. (\ref{eq:logISStrling2}) becomes:
\begin{equation}
\lim_{N \rightarrow \infty} \log  P(D_o|\vars Z) = \lim_{N \rightarrow \infty} \sum_{j=1}^{q} \sum_{k=1}^{r}  N_{jk} \log \frac{N_{jk}}{N_{j}} + \frac{1}{2} \sum_{j=1}^{q}  \Biggl[ \log (\alpha_{j}+N_{j}) - \sum_{k=1}^{r}  \log (\alpha_{jk}+N_{jk})  \Biggl] + const.  ,
\end{equation}
or equivalently:
\begin{equation}\label{eq:logISStrlinbf}
\begin{split}
\lim_{N \rightarrow \infty} \log P(D_o|\vars Z) &=
\lim_{N \rightarrow \infty} N \cdot \sum_{j=1}^{q} \sum_{k=1}^{r}  \frac{N_{jk}}{N} \log \frac{N_{jk}}{N_{j}} + \frac{1}{2} \sum_{j=1}^{q}  \Biggl[ \log (\alpha_{j}+N_{j}) - \sum_{k=1}^{r}  \log (\alpha_{jk}+N_{jk})  \Biggl] + const.\\
  & =\lim_{N \rightarrow \infty} -N \cdot H(Y|\vars Z) + \frac{1}{2} \sum_{j=1}^{q}  \Biggl[ \log (\alpha_{j}+N_{j}) - \sum_{k=1}^{r}  \log (\alpha_{jk}+N_{jk})  \Biggl] + const.
\end{split}
\end{equation}
To simplify the second term in Eq. (\ref{eq:logISStrlinbf}), we divide the arguments in the log terms by $N$ and equivalently add $\log N$ terms as follows:
\begin{equation}\label{eq:logISStrlingSecondTerm}
\begin{split}
\lim_{N \rightarrow \infty}\frac{1}{2}  \sum_{j=1}^{q}  \Biggl[ \log (\alpha_{j}+N_{j}) &- \sum_{k=1}^{r}  \log (\alpha_{jk}+N_{jk})  \Biggl]=\lim_{N \rightarrow \infty} \frac{1}{2} \sum_{j=1}^{q}  \Biggl[ \log (\frac{\alpha_{j}+N_{j}}{N}) + \log N - \sum_{k=1}^{r}  \log (\frac{\alpha_{jk}+N_{jk}}{N}) +\log N  \Biggl] \\
&= \lim_{N \rightarrow \infty} \frac{1}{2} \sum_{j=1}^{q}   \left( \log N - \sum_{k=1}^{r}  \log N  \right) + \frac{1}{2} \sum_{j=1}^{q}  \Biggl[ \log (\frac{\alpha_{j}+N_{j}}{N}) - \sum_{k=1}^{r} \log (\frac{\alpha_{jk}+N_{jk}}{N}) \Biggl] \\
& = -\frac{q(r - 1)}{2} \log N + const. 
\end{split}
\end{equation}
Combining  Equations (\ref{eq:logISStrlinbf}) and (\ref{eq:logISStrlingSecondTerm}), we obtain:
\begin{equation}\label{eq:logISStrling3}
\lim_{N \rightarrow \infty} \log P(D_o|\vars Z) =\lim_{N \rightarrow \infty} -N \cdot H(Y|\vars Z')  -\frac{q \cdot (r- 1)}{2} \log N + const.
\end{equation}
\end{proof}


\begin{reptheorem}{the:fgesmb}
Given dataset $D_o$ that contains samples from a strictly positive distribution $P$, which is a perfect map for a SMCM \graph G, the BD score~\citep{heckerman1995learning} will assign the highest score to the OMB of $Y$ in the large sample limit.
\end{reptheorem}

\begin{proof}
Let $\vars Z'$ be the OMB of $Y$ and $\vars Z \subseteq \vars V$ be an arbitrary set. We want to show that:
\begin{equation}\label{eq:bscd}
  \lim_{N \rightarrow \infty}\frac{P(D_o|\vars Z)}{P(D_o| \vars Z')} = \begin{cases}
             1 & \text{iff $\vars Z$ is an OMB of $Y$}\\
             0 & \text{otherwise} 
          \end{cases}  \enspace ,
\end{equation}
Applying Lemma \ref{lem:BDMB} we have:
\begin{equation}\label{eq:ratio}
\lim_{N \rightarrow \infty}\log \frac{P(D_o|\vars Z)}{P(D_o| \vars Z')} = \lim_{N \rightarrow \infty}N \cdot \left[ H(Y|\vars Z')- H(Y|\vars Z)\right]  + \frac{(q^{\prime} - q) \cdot (r-1)}{2} \log N .
\end{equation}
where $q$ and $q'$ are the number of possible parent instantiations of $Y$ with $\vars Z$ and $\vars Z'$ as the set of parents. There are three possible cases:\\

\textbf{Case 1:} $\vars Z$ is a Markov blanket of $Y$ and its OMB.

Since both $\vars Z'$ and $\vars Z$ are Markov blankets of $Y$, $H(Y|\vars Z)=H(Y|\vars Z')$ by Lemma~\ref{lem:entMB}. Thus, the first term in Eq. (\ref{eq:ratio}) becomes $0$.
Also, given that $\vars Z'$ and $\vars Z$ are OMBs, they have the same number of parameters $q'= q$, by which the second term in Eq. (\ref{eq:ratio}) becomes $0$ in the limit as $N \rightarrow \infty$, or equivalently Eq. (\ref{eq:bscd}) approaches to 1.
\\

\textbf{Case 2:} $\vars Z$ is a Markov blanket of $Y$ but not its OMB.

According to Lemma~\ref{lem:entMB} $H(Y|\vars Z)=H(Y|\vars Z')$; therefore, the first term in Eq. (\ref{eq:ratio}) becomes $0$ and we obtain:
\begin{equation}\label{eq:case2}
\lim_{N \rightarrow \infty}\frac{P(D_o|\vars Z)}{P(D_o| \vars Z')} = \lim_{N \rightarrow \infty} \frac{(q^{\prime} - q) \cdot (r-1)}{2} \log N .
\end{equation}
Given that $\vars Z'$ is the OMB with minimum number of variables, and therefore, minimum number of parameters $q^{\prime} < q$. Thus, the term $(q^{\prime}-q)$ becomes a negative constant. Also, the term $\frac{(r-1)}{2}$ is a positive constant. Consequently, Eq. (\ref{eq:case2}) goes to $-\infty$ in the limit as $N \rightarrow \infty$, which implies that Eq. (\ref{eq:bscd}) approaches to 0.
\\

\textbf{Case 3:} $\vars Z$ is not a Markov blanket of $Y$.

The first term in Eq. (\ref{eq:ratio}) is of $O(N)$ and dominates the second term, which is  $O(\log N)$. 
According to Lemma \ref{lem:entZneqZ'}, $H(Y|\vars Z') < H(Y|\vars Z)$; thus, the term $H(Y|\vars Z') -H(Y|\vars Z)$ becomes a negative number. As a result, Eq. (\ref{eq:ratio}) becomes $- \infty$, which equivalently implies that Eq. (\ref{eq:bscd}) becomes 0.
\end{proof}

\end{document}